\documentclass{article}
\usepackage[preprint]{colm2026_conference}

\usepackage{microtype}
\usepackage{hyperref}
\usepackage{url}
\usepackage{booktabs}
\usepackage{graphicx}
\usepackage{amsmath,amsfonts,amssymb}
\usepackage{amsthm}
\usepackage{bm}
\usepackage{multirow}
\usepackage{xcolor}
\usepackage{subcaption}
\usepackage{lineno}
\usepackage{pgfplots}
\pgfplotsset{compat=1.18}
\usepgfplotslibrary{groupplots,fillbetween}
\usetikzlibrary{calc}

\newtheorem{theorem}{Theorem}

\newtheorem{assumption}{Assumption}

\definecolor{darkblue}{rgb}{0, 0, 0.5}
\hypersetup{colorlinks=true, citecolor=darkblue, linkcolor=darkblue, urlcolor=darkblue}

\newcommand{\GMI}{\mathrm{GMI}}
\newcommand{\Llog}{L_{\log}}
\newcommand{\PT}{P_{\!T}}
\newcommand{\PM}{P_{\!M}}
\newcommand{\Wone}{W_{\!1}}
\newcommand{\E}{\mathbb{E}}
\newcommand{\R}{\mathbb{R}}

\title{Modality Collapse as Mismatched Decoding: \\ Information-Theoretic Limits of Multimodal LLMs}

\author{Jayadev Billa\thanks{Unaffiliated researcher; previously at ISI@USC, Yahoo, Nuance, and BBN.} \\
\texttt{jbilla2004@gmail.com}}

\begin{document}

\ifcolmsubmission
\linenumbers
\fi

\maketitle

\begin{abstract}
Numerous studies have shown that multimodal LLMs process speech and images well but fail in non-intuitive ways rendering trivial tasks such as object counting unreliable. We investigate this behavior from an information-theoretic perspective by framing multimodal LLM inference as a \emph{mismatched decoder} problem: a decoder trained primarily on text can only extract information along text-aligned directions (removing up to 98\% of the variation in modality-specific (non-text) directions \emph{improves} decoder loss) and the amount of accessible information is bounded by the Generalized Mutual Information (GMI). We show that information loss is bounded as the distributional mismatch between the source data and the text data increases, and as the sensitivity of the decoder increases. This bound is a function of the model's scoring rule not its architecture. 
We validate the predictions across five models spanning speech and vision. A controlled study (two Prismatic VLMs differing only in encoder text-alignment) shows that the bottleneck lies in the scoring rule of the decoder rather than the text-alignment of the encoder or the learned projection. A LoRA intervention demonstrates that simply training with an emotion-related objective improves emotion detection from 17.3\% to 61.8\% task accuracy without affecting other attributes, confirming that the training objective determines what becomes
accessible.

\end{abstract}

\ifcolmsubmission\else
\noindent\url{https://github.com/jb1999/modality_collapse_paper}
\fi

\section{Introduction}
\label{sec:intro}

Multimodal LLMs follow a common recipe: an encoder processes the non-text input (speech, image), a learned projection maps it into the LLM's embedding space, and the LLM generates a text response~\citep{Liu2023, Dai2023, Fixie2024, Chu2024, Karamcheti2024}. These systems appear to perform well on standard benchmarks. However, there are many reported gaps in model performance~\citep{Chen2025listen,Cuervo2026salad,Thrush2022,Yuksekgonul2023,Jain2025} where even some trivial tasks are not resolved. A model might produce a perfect response to speech input but have no understanding of the emotional state of the speaker; it may correctly identify every object in an image and yet fail to count the number of objects or understand the spatial relationships. Here, the information is clearly available to the model, since it can list every object, but it cannot make use of this available information for all tasks. We refer to this selective failure as \emph{modality collapse}.

To provide insight into this behavior we borrow the well-defined concept of a \emph{mismatched decoder} from communication theory, where a decoder originally designed for source A receives source B. A multimodal LLM fits this definition of a mismatched decoder exactly: an LLM trained to "decode" text representation is now expected to extract information from audio and/or image projections. Our mismatched decoder analysis is a formalization of the perception constraint cost~\citep{Blau2019}.

In this framing the accessible information is bounded by the Generalized Mutual Information (GMI) not standard mutual information, as would be the case for matched decoders. We refer to this difference as the \emph{information accessibility gap}. The gap is not about architecture; it is about the scoring rule. No explicit adapter module is required: any system where a text-trained decoder processes non-text representations faces the same constraint, whether the representations arrive through a learned projection, a discrete codebook, or the LLM's own layers. At inference, every trained model has fixed weights: the scoring rule is whatever the training procedure produced. A model trained predominantly on text has a text-shaped scoring rule regardless of whether the LLM was frozen or fine-tuned during training. What matters is the training objective, not which components were updated (\S\ref{sec:implications}).

What causes this gap? A transformer model is sensitive to its input in all directions; one trained on text is still sensitive to non-text directions (we validate this empirically using gradient isotropy in \S\ref{sec:bound_validation}). That said, its sensitivity to non-text directions is unproductive. The decoder treats unfamiliar structure as noise, and noise hurts its text processing, as would be expected in a \emph{mismatched decoder}. In \S\ref{sec:causal_ablation} we empirically show that removing up to 98\% of the variation in modality-specific (non-text) directions \emph{improves} decoder loss; the decoder is not indifferent to non-text directions, but responds to them in destructive ways. We can quantify this degradation (Theorem~\ref{thm:gmi_wasserstein}). It is proportional to both the distributional distance between the modality and text representations and the decoder's sensitivity to this distance (its Lipschitz constant $\Llog$, which measures how much the output changes per unit change in the input). This sensitivity can be up to $30\times$ that of a linear probe (a classifier trained on frozen representations to test what information is present; Theorem~\ref{thm:probe}, \S\ref{sec:asymmetry}). Text-aligned encoders reduce the distance (\S\ref{sec:controlled}), but only by discarding non-textual information earlier on. The collapse occurs earlier, not less.

\textbf{Contributions.} \textbf{(1)}~We formalize modality collapse as mismatched decoding, and prove that the accessible information is at most the GMI, and degrades with the distributional mismatch and the sensitivity of the decoder (\S\ref{sec:theory}). \textbf{(2)}~On five models and two modalities, we show that the information accessibility gap is present: non-text information is preserved by the LLM but is not decodable because the decoder is not incentivized to use it (\S\ref{sec:causal_ablation}, \S\ref{sec:deep_dive}). \textbf{(3)}~We run a controlled study on Prismatic VLMs (identical architecture, identical LLM, only text-alignment of the encoder changes), and show that the scoring rule is the causal factor (\S\ref{sec:controlled}). \textbf{(4)}~We run a LoRA intervention that confirms the prescription: training with an emotion objective improves emotion task accuracy from 17.3\% to 61.8\% ($+$7.5\% probe), while training with a text objective does not (\S\ref{sec:implications}).

\section{Related work}
\label{sec:related}

\textbf{Multimodal LLMs and the adapter paradigm.}
The practice of projecting encoded representations into the embedding space of an LLM has been widely adopted for both vision~\citep{Liu2023, Dai2023, Zhu2024, Karamcheti2024, Tong2024} and speech~\citep{Fixie2024, Chu2024}. While these approaches perform well on text-centric tasks, they perform poorly on tasks that require modality-specific knowledge, a phenomenon observed empirically but not understood theoretically.

\textbf{Modality gap.}
\citet{Liang2022} observed that the CLIP embeddings of different modalities are separated by a ``modality gap'' in the shared space. More recently, \citet{Huh2024} reported that the representations of different modalities become similar. While our results are largely orthogonal to these geometric findings, we provide an information-theoretic explanation for the modality gap: it is not that information is lost, but rather that the decoder is unable to make use of information in unfamiliar directions.

\textbf{Modality collapse.} \citet{Javaloy2022} use this term to describe models that discard an entire modality during training due to gradient conflicts. Our usage describes a subtler failure: the model processes all modalities and preserves their information internally, but the decoder selectively fails to use non-primary-modality content.

\textbf{Mismatched decoding.} The mismatched decoder framework~\citep{Merhav1994, Lapidoth1996, Scarlett2020} provides the data rates supported by a decoder mismatched to its incoming channel. We borrow this approach and apply it to the multimodal LLM framework. At inference time, any LLM that uses a scoring rule that was learned via a text-dominant training procedure is a mismatched decoder when applied to non-text inputs. The mismatch arises from the loss function, not from what parameters were updated.

\textbf{Representation probing.} Linear probes are a common method for probing the information content of neural representations~\citep{Alain2017, Conneau2018, Hewitt2019, Belinkov2022}. We do not use probing to argue that some information is ``encoded'' in a representation, but to highlight the difference between information that is \emph{present} (probe-recoverable) and information that is \emph{accessible} (decoder-extractable).

\textbf{Rate-distortion-perception (RDP) in ML.}
\citet{Blau2019} formalized perceptual quality as a distributional constraint: reconstructions are perceptually good when their distribution matches the real data distribution. The text-trained decoder imposes a perception constraint (text-aligned directions), forcing the adapter to compromise modality-specific information to satisfy the constraint. We draw on this framework in \S\ref{sec:formulation}, where the decoder takes the role of the perceptual evaluator.

\section{Problem formulation}
\label{sec:formulation}

\textbf{Architecture.}
We consider the typical multimodal LLM pipeline:
\begin{equation}
X \xrightarrow{E} H \xrightarrow{A} Z \xrightarrow{L(\cdot, C)} \hat{Y}
\label{eq:pipeline}
\end{equation}
where $X$ is a non-text input (speech waveform or image), $E$ is a pre-trained encoder (e.g., Whisper~\citep{Radford2023}, CLIP~\citep{Radford2021}), $A$ is a learned adapter, $Z \in \R^d$ are the hidden states consumed by the LLM $L$, $C$ is the text context (prompt), and $\hat{Y}$ is the generated output. The specific form of $A$ varies (linear, MLP, Q-Former~\citep{Dai2023}, perceiver, codebook), but the mismatch depends on the scoring rule, not the adapter type (\S\ref{sec:implications}).

\textbf{Scoring rule, text law and modal law.}
Let $q(y|z,c)$ denote the LLM's next-token distribution at inference. This is the fixed scoring rule resulting from training. During pre-training, the LLM processes hidden states drawn from the text law $\PT(C, Z, Y)$. When processing adapted non-text inputs, the hidden states follow a different modal law $\PM(C, Z, Y)$.

We assume the text and modal laws share the same prompts and target tokens ($\PM(C,Y) = \PT(C,Y)$), so the only difference is in the representations $Z$. The quantitative bound (\S\ref{sec:theory}) additionally requires that the decoder's log-score is Lipschitz continuous and that representations occupy a bounded region  (Assumptions~\ref{ass:lipschitz},~\ref{ass:bounded} in Appendix~\ref{app:bound}, validated empirically in \S\ref{sec:bound_validation}).

\textbf{Formalizing modality collapse.}
The earlier intuitive explanation of modality collapse (information is available but the decoder cannot use it) can be stated precisely. For an attribute $S_\tau$
  of type $\tau$ (e.g., speaker identity, emotion), we define the \emph{information accessibility gap} as:
\begin{equation}
\Delta_{\text{access}}(\tau) = I(Z; S_\tau) - \GMI_{\PM}(S_\tau \mid q)
\label{eq:accessibility_gap}
\end{equation}
where $I(Z; S_\tau)$ is the information about $S_\tau$ present in the representation and $\GMI_{\PM}(S_\tau | q)$ is the information actually extractable by the decoder's fixed scoring rule. When $\Delta_{\text{access}} = 0$, the decoder extracts everything present; when $\Delta_{\text{access}}$  is large, information is present but inaccessible.

\textbf{Implicit RDP tradeoff.}
\citet{Blau2019} formalized perceptual quality as a distributional constraint: reconstructions are perceptually good when their distribution matches the real data distribution, measured as a divergence $d(p_X, p_{\hat{X}})$. In our setting, the perceptual evaluator is the decoder itself. A text-trained decoder implicitly judges representations by how closely they resemble its training distribution $\PT$; representations drawn from a different distribution $\PM$ are treated as low quality regardless of their information content. When $\PM \neq \PT$, the adapter must sacrifice modality-specific content (distortion) to produce representations the decoder can process. We measure this perception cost using the Wasserstein distance $\Wone(\PM, \PT)$, chosen because it aligns with the Lipschitz continuity of the decoder to yield the GMI bound (\S\ref{sec:theory}); other divergences would capture the same mismatch but without this property. The role of adapter capacity (the rate dimension of the tradeoff) is a direction for future work.

\textbf{Text-aligned encoders as correlation mappings.}
Contrastive encoders like CLIP~\citep{Radford2021} and SigLIP~\citep{Zhai2023} do not merge the modality manifolds (the modality gap persists~\citep{Liang2022}), but they align the \emph{informative directions}: attributes that co-occur with text descriptions occupy dimensions also used by text representations. When such an encoder feeds an LLM, the modal law $\PM$ has high overlap with $\PT$ along informative dimensions, so the distributional distance is small and the accessibility gap is correspondingly small (we formalize this in terms of the Wasserstein distance in \S\ref{sec:bound}). Non-contrastive encoders like DINOv2~\citep{Oquab2024} lack this alignment: their informative directions may be orthogonal to $\PT$, making the information effectively inaccessible to the decoder. We return to this in \S\ref{sec:controlled}, where the Prismatic comparison provides direct evidence.

\section{Theoretical framework}
\label{sec:theory}

The mismatched decoder framing makes three concrete predictions. First, since the decoder was trained on text, there should be an upper limit on how much non-text information it can extract, regardless of how much is present in the representation. Second, this limit should get worse as the modality representations diverge from text. Third, a simple classifier (a linear probe) trained directly on the representations should not be subject to the same limit, since it is not constrained by a text-trained scoring rule. We formalize each below (formal statements, assumptions, and proofs in Appendix~\ref{app:theory}):

\textbf{1.\ The decoder has an upper bound.}  The standard mutual information $I(Z;Y)$ represents the information $Z$ contains about $Y$ assuming an optimal decoder. But the LLM is not an optimal decoder for non-text inputs, given a text-focused training objective. Instead, the relevant quantity is the Generalized Mutual Information (GMI)~\citep{Merhav1994, Scarlett2020}, the maximum rate achievable with a fixed scoring rule. No matter how much information any probe can get from $Z$, the decoder is always bounded by $\GMI_{\PM}(q)$ (Theorem~\ref{thm:accessible_rate}).

\textbf{2.\ The ceiling drops with distributional distance.}  The GMI decreases when the modal representations $\PM$ differ from the text representations $\PT$ the decoder was trained on. How much it drops depends on two measurable quantities: how far the modal representations are from text (the Wasserstein distance $\Wone$), and how sensitive the decoder is to that difference (the Lipschitz constant $\Llog$). The drop is bounded by their product (Theorem~\ref{thm:gmi_wasserstein}). Both quantities can be estimated from a trained model ($\Wone$ from the representations, $\Llog$ from gradient norms), so the bound yields a testable prediction: models with
larger $\Llog \cdot \Wone$ should show more degradation. We validate this ranking in \S\ref{sec:bound_validation}. In the high-mismatch regime the bound is numerically loose, but it remains an ordinal diagnostic: it correctly ranks model--modality pairs by collapse severity and predicts which information types degrade (Appendix~\ref{app:support} tightens the prefactor from ${\sim}e^{150}$ to ${\sim}e^{7}$).

\textbf{3.\ Probes and decoders respond differently to the same shift.}  For a linear probe, the effect of a distributional shift is bounded by the (small) Lipschitz constant $L_h$ (Theorem~\ref{thm:probe}). For the decoder, it is bounded by $\Llog$, which is empirically 30$\times$ larger. A $\Wone$ shift of the same magnitude therefore results in a small degradation for probes and a large degradation for the decoder. This is the formal basis for the accessibility gap: the same information can be simultaneously recoverable by a probe and unusable by the decoder, because the two have vastly different sensitivity to distributional mismatch.

The three propositions above establish the perception-distortion face of the RDP tradeoff: as the distributional distance $\Wone(\PM, \PT)$ grows, accessible information (GMI) decreases, and the decoder is disproportionately more affected than a probe. 

\section{Experimental validation}
\label{sec:experiments}

\subsection{Setup}
\label{sec:setup}

\textbf{Models.}
We investigate five multimodal LLMs across two modalities (Table~\ref{tab:models}). For speech: \textbf{Ultravox-v0.6}~\citep{Fixie2024} (Whisper encoder, MLP adapter with SwiGLU, Llama-3.1-8B) and \textbf{Qwen2-Audio-7B}~\citep{Chu2024} (Whisper encoder, linear adapter, Qwen2-7B). For vision: \textbf{LLaVA-v1.5-7B}~\citep{Liu2023} (CLIP-ViT, MLP adapter, Vicuna-7B) and two \textbf{Prismatic VLMs}~\citep{Karamcheti2024} with the same MLP adapter and Vicuna-7B backbone but with different vision encoders: \textbf{DINOv2} (not text-aligned) and \textbf{SigLIP} (text-aligned via contrastive loss).

\begin{table}[t]
\centering
\small
\begin{tabular}{llll}
\toprule
\textbf{Model} & \textbf{Encoder} & \textbf{Text-aligned?} & \textbf{LLM} \\
\midrule
Ultravox & Whisper & No & Llama-3.1-8B \\
Qwen2-Audio & Whisper & No & Qwen2-7B \\
LLaVA & CLIP-ViT & Yes & Vicuna-7B \\
Prismatic-D & DINOv2 & No & Vicuna-7B \\
Prismatic-S & SigLIP & Yes & Vicuna-7B \\
\bottomrule
\end{tabular}
\caption{Models studied. Prismatic-D and Prismatic-S share the same architecture, adapter, training recipe, and LLM backbone; only the vision encoder differs.}
\label{tab:models}
\end{table}

\textbf{Datasets.}
Speech: LibriSpeech~\citep{Panayotov2015} (2,620 samples, 50 words, 40 speakers), CREMA-D~\citep{Cao2014} (7,442 samples, 6 emotions, 91 speakers), ESC-50~\citep{Piczak2015} (2,000 samples, 50 sound classes). Vision: MS-COCO~\citep{Lin2014} (5,000 images, 69 object categories, 12 super-categories).

\textbf{Probing protocol.}
We consider four hook points per model: encoder output, adapter output, LLM layer 16, and LLM final layer (after LayerNorm). Representations are mean-pooled across sequence positions. Ultravox also provides audio token boundaries, but pooling over audio-only tokens results in nearly identical trends (see Appendix Table~\ref{tab:audio_only}). For each hook point, information type and dataset combination, we train $\ell_2$-regularized logistic regression probes (5 seeds, 80/20 stratified split, $z$-score normalization).

\textbf{Information types.}
Speech: \emph{lexical} (word identity), \emph{speaker} (speaker identity), \emph{emotion} (6-way), \emph{acoustic} (sound class). Vision: \emph{lexical} (caption-based word prediction), \emph{object category} (69-way), \emph{super-category} (12-way).

The Prismatic pair (shared architecture, different encoder) and the two speech models (shared encoder, different LLM) serve as controls that allow us to isolate these specific variables.

\subsection{The information accessibility gap in practice}
\label{sec:deep_dive}


Linear probes across layers show that for all 5 models, non-text information is not destroyed by the LLM  (Appendix Tables~\ref{tab:full_probes} and~\ref{tab:retention}): at the final layer, speaker identity is $22\times$ above chance for Ultravox and object categories are $55\times$ above chance in Prismatic-D ($.792$ vs.\ $.014$). The information is preserved. Does the decoder use it?

The answer depends on text-alignment. In both speech models, the LLM approximately doubles lexical information ($+$92\% Ultravox, $+$95\% Qwen2-Audio), but degrades speaker identity ($-$8\% to $-$39\%). When the encoder is text-aligned (CLIP in LLaVA, SigLIP in Prismatic-S), $\PM \approx \PT$ and all information types improve ($+$5--8\%). The decoder amplifies what it was trained on and is indifferent to the rest. The LISTEN benchmark~\citep{Chen2025listen} independently observed this pattern: audio LLMs predict emotions based on what the speaker \emph{says} rather than \emph{how} they say it. Our framework explains why: the text-shaped scoring rule amplifies lexical content ($+$92--95\%) while remaining indifferent to acoustic features, so emotion predictions are driven by word choice rather than prosody.

The same pattern appears in vision, though in a weaker mismatch regime (Appendix Table~\ref{tab:prismatic}). Three visual features not present in text (object count, object size, spatial distribution) show flat or modest improvement through the LLM ($\Llog \cdot \Wone \approx 13$--$54$ vs.\ $162$ for speech), consistent with stagnation unlike the sharp degradation in speech.

\subsection{Mode alignment}
\label{sec:mode_alignment}

The bound describes how much information is lost; mode alignment tells us the directions in which it is lost. We perform a principal component analysis of the adapter's output representations under the modal distribution, yielding a covariance matrix $\Sigma_M$ with eigenvectors $u_k$ and eigenvalues $\lambda_k$, sorted by decreasing eigenvalue. We refer to each eigenvector as a \emph{mode}: Mode~0 is the direction along which the adapter's output varies the most, Mode~1 the next, and so on. For each mode, we ask: has the decoder seen variation along this direction during training? The alignment score $$\rho_k = u_k^\top \Sigma_T\, u_k \,/\, \lambda_k$$ is the Rayleigh quotient of the text covariance $\Sigma_T$ evaluated along the modal eigenvector $u_k$, normalized by the modal eigenvalue $\lambda_k$. The numerator measures how much the text distribution varies along direction $u_k$; the denominator is how much the modality varies along the same direction. When $\rho_k = 1$, the text and modal distributions have equal variance along $u_k$, so the decoder has been trained on exactly this kind of variation and can use it. When $\rho_k \approx 0$, the modality varies along a direction where text has almost no variance: the decoder has never seen input spread along this direction and cannot productively respond to it. We label modes with $\rho_k < 0.5$ as modality-specific (MS); the threshold is a practical cutoff, and results are not sensitive to the exact value (Appendix Table~\ref{tab:mode_alignment}).

For non-aligned speech encoders, Mode~0 (the direction of largest adapter variance) is modality-specific ($\rho_0 \leq 0.034$), accounting for 51--79\% of adapter variance. For Prismatic-DINOv2, Mode~0 is also modality-specific ($\rho_0 = 0.034$) but carries only 8\% of variance; the MS variance is instead distributed across 53 modes (71\% total; Table~\ref{tab:mode_alignment}). In both cases, the dominant source of variation in the adapter's output is invisible to the text distribution. For SigLIP, even Mode~0 is text-aligned ($\rho_0 = 0.83$): the encoder has already organized its output along text-like directions. As representations pass through the LLM, text-aligned variance is amplified (${\sim}150\times$) while MS variance is unaffected: MS modes go from 63.6\% of adapter variance to $<$1\% at the final layer in Ultravox, not destroyed but drowned out.

\subsection{Modality-specific structure interferes with decoding}
\label{sec:causal_ablation}

Probes show that non-text information is present; causal ablation tests whether the decoder is impacted by it. Unlike concept erasure methods such as LEACE~\citep{Belrose2023}, which eliminate directions predictive of a labeled attribute, we ablate directions identified by distributional mismatch without concept labels. For each mode classified as MS ($\rho_k < 0.5$) in \S\ref{sec:mode_alignment}, we remove its contribution from the representation: $z' = z - \sum_{k \in S}(z^\top u_k),u_k$, and compare three conditions (200 samples per model): (i)~ablate all MS modes, (ii)~ablate a matched number of top TA modes, (iii)~no ablation.

\begin{table}[t]
\centering
\small
\begin{tabular}{llccc}
\toprule
\textbf{Model} & \textbf{Condition} & \textbf{Modes} & \textbf{Var.\ \%} & \textbf{$\Delta$Loss (\%)} \\
\midrule
\multicolumn{5}{@{}l}{\emph{Non-aligned encoders}} \\[2pt]
Ultravox     & MS (all)     & 11  & 63.6  & $-$1.4 \\
             & TA-matched   & 11  & 29.4  & $-$0.3 \\
             & Random       & 11  & 2.8   & $+$0.3 \\
Qwen2-Audio  & MS (all)     & 5   & 97.6  & $-$0.1 \\
             & TA-matched   & 5   & 1.8   & $+$1.0 \\
             & Random       & 5   & 0.1   & $+$0.2 \\
Prismatic-D  & MS (all)     & 53  & 71.0  & $-$11.1 \\
             & TA-matched   & 47  & 29.0  & $-$0.07 \\
             & Random       & 53  & 39.1  & $-$5.1 \\
\midrule
\multicolumn{5}{@{}l}{\emph{Text-aligned encoder}} \\[2pt]
Prismatic-S  & MS (all)     & 14  & 24.6  & $-$0.5 \\
             & TA-matched   & 14  & 36.2  & $-$0.7 \\
             & Random       & 14  & 10.6  & $-$0.2 \\
\bottomrule
\end{tabular}
\caption{Causal ablation at the adapter output. Eigenmodes are classified as modality-specific (MS; $\rho_k < 0.5$) or text-aligned (TA), and matched subsets are projected out. Random controls ablate the same number of modes chosen uniformly at random (seed-averaged). For non-aligned encoders, removing MS modes improves loss while removing equal-sized TA subsets has negligible effect; the contrast is $160\times$ for Prismatic-D ($t = -25.2$ vs.\ $t = -0.4$). Random ablation gives intermediate effects proportional to variance removed. For SigLIP, few modes are MS and all effects are ${<}1$\%.}
\label{tab:causal_ablation}
\end{table}

\textbf{MS structure is not merely unused; its presence degrades decoding.}
Removing variance from a representation should hurt performance since there is strictly less information available. Yet for non-aligned encoders, removing MS modes improves decoder loss (Table~\ref{tab:causal_ablation}). Qwen2-Audio provides the cleanest demonstration: 5 MS modes carry 97.6\% of adapter variance, yet removing them improves loss ($-$0.1\%, $t = -9.0$), while removing 5 TA modes (only 1.8\% of variance) hurts 7$\times$ more ($+$1.0\%, $t = 27.8$). Nearly all of the adapter's variance is dispensable; a tiny fraction carrying text-aligned information is essential. Prismatic-D confirms the pattern at larger scale: ablating 53 MS modes (71\% of adapter variance) decreases cross-entropy by 11.1\% ($t = -25.2$), while ablating a matched number of TA modes has no detectable effect ($-$0.07\%, $t = -0.4$, $p = 0.69$). For Ultravox the effect is smaller ($-$1.4\%) but still significant ($t = -7.0$). The decoder is not just indifferent to MS directions, it is actively harmed by them.

For Prismatic-S (text-aligned), only 14 of 100 modes are MS (24.6\% of variance) and all ablation effects are below 1\%, consistent with SigLIP representations already matching what the decoder expects.

\subsection{Controlled experiment: text-aligned vs.\ non-aligned encoders}
\label{sec:controlled}

The causal ablation shows that MS directions influence decoding. Now, we explain \emph{why} some models have more MS variability than others. The Prismatic pair is ideal for this, as it controls for all factors except the text-alignment of the encoder.

Both Prismatic VLMs use the same MLP adapter, Vicuna-7B backbone architecture and training protocol. The difference lies in the vision encoder: DINOv2~\citep{Oquab2024} (no text alignment) versus SigLIP~\citep{Zhai2023} (contrastive text--image loss). According to our hypothesis, the representations learned by SigLIP should be more similar to $\PT$, leading to smaller $\Wone(\PM, \PT)$ and less information loss.

This is what we see (Appendix Table~\ref{tab:prismatic}): the split is sharper for non-textual attributes (object count, object size, spatial spread), which cannot be inferred from captions alone. For these features, SigLIP shows consistent improvement through the LLM (+4.5--8.2\%) while DINOv2 gains less (+1.3--4.7\%), with the largest gap on object count (+8.2\% vs.\ +3.4\%).

For DINOv2, visual information (object category, super-category) passes through the LLM unchanged ($\pm0.4$\%), neither amplified nor degraded. For SigLIP, the same information is amplified (+1.9\% object category, +2.3\% super-category). The contrast confirms the prediction: a text-aligned encoder produces representations the decoder can productively process; a non-aligned encoder produces representations the decoder is indifferent to. For LLaVA (CLIP encoder, text-aligned), all information types improve (Appendix Table~\ref{tab:retention}). The effect sizes in vision are modest ($\pm$2\%) compared to speech ($\pm$39\%), consistent with the lower mismatch regime ($\Llog \cdot \Wone$ is 3--12$\times$ smaller for vision). The value of the Prismatic comparison lies not in the magnitude but in the isolation of the causal variable: same architecture, same adapter, same LLM, same training with the only difference being the encoder's text-alignment.

\subsection{Validating the bound}
\label{sec:bound_validation}

$\Llog$ is finite for all nine model-dataset pairs, validating the Lipschitz assumption (Appendix Table~\ref{tab:lipschitz}). It varies with content ($9.08$ for speech to $3.66$ for environmental sounds in Ultravox) and differs $30\times$ across architectures ($0.29$ LLaVA to $9.08$ Ultravox). Prismatic-S has $2.8\times$ larger $\Llog$ than Prismatic-D ($1.12$ vs.\ $0.40$): the decoder is more responsive to representations that resemble text.

\textbf{Gradient isotropy.} To check whether the decoder selectively ignores MS directions at the gradient level, we compute the norm of the decoder's loss gradient when the input representation is perturbed along each eigenmode direction. Let $g_k = |\nabla_{u_k} \mathcal{L}|$ be the gradient norm for mode $k$. The mean gradient norm for TA modes ($\bar{g}_{\text{TA}}$) and MS modes ($\bar{g}_{\text{MS}}$) are nearly equal ($\bar{g}_{\text{TA}} / \bar{g}_{\text{MS}} = 0.94$, Spearman $r = 0.11$, $p = 0.28$): the decoder does not selectively ignore MS directions. This confirms the scalar $\Llog$ assumption.

\textbf{The product $\Llog \cdot \Wone$ ranks models by collapse severity.} Ultravox ($\Llog \cdot \Wone = 162$) shows the most degradation ($-$39\% speaker); LLaVA ($13.4$) shows none. Prismatic-S has a higher product ($53.7$) than Prismatic-D ($20.2$) despite less degradation, reflecting greater decoder responsiveness to text-aligned representations rather than greater mismatch. As noted in \S\ref{sec:theory}, the bound is an ordinal diagnostic: it correctly ranks models and predicts which information types degrade, even in the high-mismatch regime where it is numerically loose (Appendix~\ref{app:support},~\ref{app:per_type}).

\section{Discussion}
\label{sec:discussion}

\subsection{Implications}
\label{sec:implications}

\textbf{The fix is objective-side, not encoder-side.}
If modality collapse is a result of the scoring rule being text-shaped, a logical solution would be to reshape the scoring rule.  Low-rank adaptation (LoRA;~\citealp{Hu2022}) on the LLM modifies $q$ to make it closer to $\PM$, effectively reducing the functional mismatch. Preliminary evidence from concurrent work on audio-text conflict supports this: applying LoRA to the LLM backbone decreases text dominance, whereas adapter-only training \emph{increases} it.\footnote{Citation omitted for double-blind review.} This is predicted by Theorem~\ref{thm:gmi_wasserstein}: LoRA reduces the functional mismatch between $\PM$ and $\PT$ by reshaping the scoring rule, whereas adapter-only training does not change the decoder's text-shaped sensitivity.

However, LoRA on the decoder is \emph{necessary but not sufficient}: the training objective must explicitly target non-text information. We show this with a LoRA experiment on Ultravox ($r{=}16$, $\alpha{=}32$, targeting \texttt{q/k/v/o\_proj}).

\textbf{LoRA with emotion objective.} We fine-tune with LoRA on the LLM backbone using a forced-choice emotion detection objective on CREMA-D (6~emotions, 7{,}442 clips, standard causal LM loss on assistant tokens only). We evaluate the effect of LoRA on two measures: linear probe accuracy at the LLM's final layer (measuring what information is accessible in the representation) and task accuracy (measuring whether the decoder actually uses it in its output). Task accuracy on held-out samples improves from 17.3\% to 61.8\% (Table~\ref{tab:lora_ablation}). The \emph{emotion probe} at llm-final improves from $.557 \pm .006$ to $.632 \pm .011$ ($+$7.5\%), while the speaker probe remains at chance ($.135 \pm .006$ vs.\ $.137 \pm .008$) and lexical accuracy stays near-ceiling ($.994 \pm .001$).
Upstream layers (encoder, adapter) are shared between conditions, so this is purely a change to the LLM. LoRA modifies the scoring rule such that at inference, the decoder will respond to the subset of directions that are emotion-relevant and leave the rest unchanged. Consistent with Theorem~\ref{thm:gmi_wasserstein}, the emotion objective reduces $\Wone$ along the subset of directions correlated with emotion, while leaving all others unchanged. 

\begin{table}[t]
\centering
\small
\caption{LoRA intervention on Ultravox (CREMA-D, llm-final probe accuracy). The emotion LoRA selectively improves emotion accessibility without meaningfully affecting speaker or lexical probes. Task accuracy measures forced-choice emotion detection on 1{,}002 held-out samples. The encoder and adapter outputs are identical across conditions (LoRA modifies only the LLM).}
\label{tab:lora_ablation}
\begin{tabular}{lcccc}
\toprule
\textbf{Condition} & \textbf{Emotion (probe)} & \textbf{Speaker (probe)} & \textbf{Lexical (probe)} & \textbf{Task acc.} \\
\midrule
Base Ultravox        & .557 & .137 & .999 & 17.3\% \\
+ Emotion LoRA       & \textbf{.632} & .135 & .994 & \textbf{61.8\%} \\
\midrule
$\Delta$             & $+$7.5\% & $-$0.2\% & $-$0.5\% & $+$44.5\% \\
\bottomrule
\end{tabular}
\end{table}

\textbf{Text-aligned encoders are a workaround, not a solution.}
Contrastive encoders like CLIP and SigLIP achieve low mismatch because they project the input data into the space of text-correlated features. The alignment is effective but selective in that it preserves what co-occurs with text and discards the rest. The Prismatic comparison shows that SigLIP's visual representations ``enhanced'' by the LLM in fact only looked better because they were pre-projected into text-like directions. CLIP provides the LLM with visual features organized along text-correlated dimensions; genuinely visual information (texture, spatial layout) is discarded upstream.

\textbf{Models must be explicitly trained to use non-text modalities.}
A text-only decoder will not listen to signals it was not trained on  (\S\ref{sec:mode_alignment}). No adapter can make it sensitive to directions it has never seen. If we want a model to use modality-specific knowledge, rather than merely its text representation, we need to provide a training signal for it to do so, whether through LoRA, full multimodal pretraining, or some other objective-side intervention such as a modality-specific loss. Otherwise the information is just not available.

\textbf{The framework is architecture-agnostic.}
The bound we present depends on the decoder's scoring rule not the model's architecture. The source of representations, a linear projection, an MLP, a Q-Former, a discrete codebook~\citep{Borsos2023, Defossez2024}, or no explicit adapter at all (i.e., the LLM's own layers act as the adapter), is immaterial. Discrete token methods introduce an additional quantization bottleneck that can only increase $\Wone(\PM, \PT)$, making our bound more conservative. The essential phenomenon is the distributional shift the decoder sees between training and inference.

\subsection{Limitations}
\label{sec:limitations}

\textbf{Probes as information measures.}
Linear probes provide a necessary but not sufficient condition for information presence~\citep{Belinkov2022}, but the main asymmetry (more information can be linearly recovered than the decoder uses) holds either way. We can only approximate $\Llog$ via p95 gradient norms, not compute the true supremum.  Other caveats (in scope, Assumption~\ref{ass:shared}) can be found in Appendix~\ref{app:additional_limitations}.

\textbf{Vision collapse is partial.}
In vision models, we observe a plateau in performance ($\pm$1--2\%) rather than a drop ($-$39\% for speech). This is consistent with the bound: in a low $\Llog \cdot \Wone$ regime (3--12$\times$ smaller than speech), the GMI ceiling is close to the matched-decoder rate, so collapse manifests as stagnation rather than degradation. The causal ablation confirms that MS structure still interferes even when probe accuracy is unchanged (Prismatic-D: $-$11.1\%, $t = -25.2$). Probing purely visual features (texture, depth, material) rather than text-describable categories would likely reveal stronger collapse.

\textbf{Scale.} All five models studied are in the 7--8B parameter range. The bound (Theorem~\ref{thm:gmi_wasserstein}) depends on the scoring rule's sensitivity ($\Llog$) and the distributional distance ($\Wone$), not on parameter count, so the theoretical constraint is scale-invariant. For larger models the bound should remain consistent with our results insofar as a text-majority-training-objective paradigm remains. Validating the bound at the 70B scale is a direction for future work.

\section{Conclusion}
\label{sec:conclusion}

Modality collapse is caused by a failure of \emph{decoding}, not encoding, as our experiments demonstrate. The GMI-Wasserstein bound shows theoretically that only information in the text-aligned directions is accessible to the decoder, and the rate at which information is lost depends on distributional divergence and decoder sensitivity.

To train multimodal models that properly exploit the information they receive, it is necessary to adjust the scoring function of the decoder. Text-aligned encoders do appear to improve performance, but they do so by ignoring all non-text modalities at the level of the encoder. Our LoRA experiment confirms this: an emotion-based objective trains the decoder to be sensitive to emotion-relevant messages ($+$7.5\% probe, $+$44.5\% task accuracy), without affecting irrelevant messages. The question, for any multimodal model, is not whether modality-specific details can be represented (they are) but whether the training objective instructs the decoder to rely on them.

\section*{Reproducibility statement}
All models are publicly available.
\ifcolmsubmission
Complete reproduction code, pipeline scripts, and expected-value regression tests will be made available upon publication.
\else
Complete reproduction code, pipeline scripts, and expected-value regression tests are available at \url{https://github.com/jb1999/modality_collapse_paper}.
\fi
Hook points, hyperparameters, and seeds are in Appendix~\ref{app:details}. All experiments run on a single NVIDIA RTX 3090 Ti (24GB).

\section*{Ethics \& AI disclosure}
This work analyses existing publicly available models and datasets; no new data was collected and no human subjects were involved. The author used Claude/Claude Code during preparation for manuscript critique, narrative feedback, literature search, and experiment implementation and debugging. All research design, theoretical development, experimental execution, analysis, and writing are the author's own. The author takes full responsibility for all content.

\bibliography{references}

@article{Merhav1994,
  author  = {Neri Merhav and
             Gideon Kaplan and
             Amos Lapidoth and
             Shlomo Shamai},
  title   = {On Information Rates for Mismatched Decoders},
  journal = {IEEE Transactions on Information Theory},
  volume  = {40},
  number  = {6},
  pages   = {1953--1967},
  year    = {1994},
}

@article{Scarlett2020,
  author    = {Jonathan Scarlett and
               Albert Guill{\'{e}}n i F{\`{a}}bregas and
               Anelia Somekh{-}Baruch and
               Alfonso Martinez},
  title     = {Information-Theoretic Foundations of Mismatched Decoding},
  journal   = {Foundations and Trends in Communications and Information Theory},
  publisher = {Now Publishers},
  volume    = {17},
  number    = {2--3},
  pages     = {149--401},
  year      = {2020},
}

@article{Lapidoth1996,
  author  = {Amos Lapidoth},
  title   = {Nearest Neighbor Decoding for Additive Non-{Gaussian} Noise Channels},
  journal = {IEEE Transactions on Information Theory},
  volume  = {42},
  number  = {5},
  pages   = {1520--1529},
  year    = {1996},
}

@inproceedings{Blau2019,
  author    = {Yochai Blau and
               Tomer Michaeli},
  editor    = {Kamalika Chaudhuri and
               Ruslan Salakhutdinov},
  title     = {Rethinking Lossy Compression: The Rate-Distortion-Perception Tradeoff},
  booktitle = {Proceedings of the 36th International Conference on Machine Learning (ICML)},
  pages     = {675--685},
  year      = {2019},
}

@inproceedings{Liang2022,
  author    = {Weixin Liang and
               Yuhui Zhang and
               Yongchan Kwon and
               Serena Yeung and
               James Y. Zou},
  editor    = {Sanmi Koyejo and
               S. Mohamed and
               A. Agarwal and
               Danielle Belgrave and
               K. Cho and
               A. Oh},
  title     = {Mind the Gap: Understanding the Modality Gap in Multi-modal Contrastive Representation Learning},
  booktitle = {Advances in Neural Information Processing Systems (NeurIPS)},
  volume    = {35},
  year      = {2022},
}

@inproceedings{Huh2024,
  author    = {Minyoung Huh and
               Brian Cheung and
               Tongzhou Wang and
               Phillip Isola},
  editor    = {Ruslan Salakhutdinov and
               Zico Kolter and
               Katherine A. Heller and
               Adrian Weller and
               Nuria Oliver and
               Jonathan Scarlett and
               Felix Berkenkamp},
  title     = {Position: The Platonic Representation Hypothesis},
  booktitle = {Proceedings of the 41st International Conference on Machine Learning (ICML)},
  volume    = {235},
  pages     = {20617--20642},
  year      = {2024},
}

@inproceedings{Javaloy2022,
  author       = {Adri{\'{a}}n Javaloy and
                  Maryam Meghdadi and
                  Isabel Valera},
  editor       = {Kamalika Chaudhuri and
                  Stefanie Jegelka and
                  Le Song and
                  Csaba Szepesv{\'{a}}ri and
                  Gang Niu and
                  Sivan Sabato},
  title        = {Mitigating Modality Collapse in Multimodal {VAE}s via Impartial Optimization},
  booktitle    = {International Conference on Machine Learning, {ICML} 2022, 17-23 July
                  2022, Baltimore, Maryland, {USA}},
  series       = {Proceedings of Machine Learning Research},
  volume       = {162},
  pages        = {9938--9964},
  publisher    = {{PMLR}},
  year         = {2022},
  url          = {https://proceedings.mlr.press/v162/javaloy22a.html},
}

@inproceedings{Belrose2023,
  author    = {Nora Belrose and
               David Schneider-Joseph and
               Shauli Ravfogel and
               Ryan Cotterell and
               Edward Raff and
               Stella Biderman},
  title     = {{LEACE}: Perfect Linear Concept Erasure in Closed Form},
  booktitle = {Advances in Neural Information Processing Systems (NeurIPS)},
  year      = {2023},
}

@inproceedings{Conneau2018,
  author    = {Alexis Conneau and
               Germ{\'{a}}n Kruszewski and
               Guillaume Lample and
               Lo{\"{\i}}c Barrault and
               Marco Baroni},
  editor    = {Iryna Gurevych and
               Yusuke Miyao},
  title     = {What You Can Cram into a Single \$\&!\#* Vector: Probing Sentence Embeddings for Linguistic Properties},
  booktitle = {Proceedings of the 56th Annual Meeting of the Association for Computational Linguistics (ACL)},
  pages     = {2126--2136},
  year      = {2018},
}

@inproceedings{Hewitt2019,
  author    = {John Hewitt and
               Christopher D. Manning},
  editor    = {Jill Burstein and
               Christy Doran and
               Thamar Solorio},
  title     = {A Structural Probe for Finding Syntax in Word Representations},
  booktitle = {Proceedings of the 2019 Conference of the North {A}merican Chapter of the Association for Computational Linguistics: Human Language Technologies,
               {NAACL-HLT}},
  pages     = {4129--4138},
  year      = {2019},
}

@article{Belinkov2022,
  author    = {Yonatan Belinkov},
  title     = {Probing Classifiers: Promises, Shortcomings, and Advances},
  journal   = {Computational Linguistics},
  volume    = {48},
  number    = {1},
  pages     = {207--219},
  year      = {2022},
}

@inproceedings{Alain2017,
  author    = {Guillaume Alain and
               Yoshua Bengio},
  title     = {Understanding Intermediate Layers Using Linear Classifier Probes},
  booktitle = {Proceedings of the 5th International Conference on Learning Representations (ICLR), Workshop Track},
  year      = {2017},
}

@inproceedings{Liu2023,
  author    = {Haotian Liu and
               Chunyuan Li and
               Qingyang Wu and
               Yong Jae Lee},
  editor    = {Alice Oh and
               Tristan Naumann and
               Amir Globerson and
               Kate Saenko and
               Moritz Hardt and
               Sergey Levine},
  title     = {Visual Instruction Tuning},
  booktitle = {Advances in Neural Information Processing Systems (NeurIPS 2023)},
  volume    = {36},
  year      = {2023},
}

@article{Chu2024,
  author  = {Yunfei Chu and
             Jin Xu and
             Qian Yang and
             Haojie Wei and
             Xipin Wei and
             Zhifang Guo and
             Yichong Leng and
             Yuanjun Lv and
             Jinzheng He and
             Junyang Lin and
             Chang Zhou and
             Jingren Zhou},
  title   = {Qwen2-{Audio}: Technical Report},
  journal = {arXiv preprint arXiv:2407.10759},
  year    = {2024},
}

@misc{Fixie2024,
  author       = {{Fixie AI}},
  title        = {Ultravox: An Open, Fast Multimodal {LLM}},
  year         = {2024},
  howpublished = {\url{https://github.com/fixie-ai/ultravox}},
}

@inproceedings{Karamcheti2024,
  author    = {Siddharth Karamcheti and
               Suraj Nair and
               Ashwin Balakrishna and
               Percy Liang and
               Thomas Kollar and
               Dorsa Sadigh},
  editor    = {Ruslan Salakhutdinov and
               Zico Kolter and
               Katherine A. Heller and
               Adrian Weller and
               Nuria Oliver and
               Jonathan Scarlett and
               Felix Berkenkamp},
  title     = {Prismatic {VLMs}: Investigating the Design Space of Visually-Conditioned Language Models},
  booktitle = {Proceedings of the 41st International Conference on Machine Learning (ICML)},
  series    = {Proceedings of Machine Learning Research},
  volume    = {235},
  pages     = {23123--23144},
  year      = {2024},
}

@inproceedings{Dai2023,
  author    = {Wenliang Dai and
               Junnan Li and
               Dongxu Li and
               Anthony Meng Huat Tiong and
               Junqi Zhao and
               Weisheng Wang and
               Boyang Li and
               Pascale Fung and
               Steven C. H. Hoi},
  editor    = {Alice Oh and
               Tristan Naumann and
               Amir Globerson and
               Kate Saenko and
               Moritz Hardt and
               Sergey Levine},
  title     = {{InstructBLIP}: Towards General-purpose Vision-Language Models with Instruction Tuning},
  booktitle = {Advances in Neural Information Processing Systems (NeurIPS 2023)},
  volume    = {36},
  year      = {2023},
}

@inproceedings{Zhu2024,
  author    = {Deyao Zhu and
               Jun Chen and
               Xiaoqian Shen and
               Xiang Li and
               Mohamed Elhoseiny},
  title     = {{MiniGPT-4}: Enhancing Vision-Language Understanding with Advanced Large Language Models},
  booktitle = {Proceedings of the 12th International Conference on Learning Representations (ICLR)},
  year      = {2024},
}

@inproceedings{Radford2023,
  author    = {Alec Radford and
               Jong Wook Kim and
               Tao Xu and
               Greg Brockman and
               Christine McLeavey and
               Ilya Sutskever},
  editor    = {Andreas Krause and
               Emma Brunskill and
               Kyunghyun Cho and
               Barbara Engelhardt and
               Sivan Sabato and
               Jonathan Scarlett},
  title     = {Robust Speech Recognition via Large-Scale Weak Supervision},
  booktitle = {Proceedings of the 40th International Conference on Machine Learning (ICML)},
  series    = {Proceedings of Machine Learning Research},
  volume    = {202},
  pages     = {28492--28518},
  year      = {2023},
}

@inproceedings{Radford2021,
  author    = {Alec Radford and
               Jong Wook Kim and
               Chris Hallacy and
               Aditya Ramesh and
               Gabriel Goh and
               Sandhini Agarwal and
               Girish Sastry and
               Amanda Askell and
               Pamela Mishkin and
               Jack Clark and
               Gretchen Krueger and
               Ilya Sutskever},
  editor    = {Marina Meila and
               Tong Zhang},
  title     = {Learning Transferable Visual Models from Natural Language Supervision},
  booktitle = {Proceedings of the 38th International Conference on Machine Learning (ICML)},
  pages     = {8748--8763},
  year      = {2021},
}

@inproceedings{Zhai2023,
  author    = {Xiaohua Zhai and
               Basil Mustafa and
               Alexander Kolesnikov and
               Lucas Beyer},
  title     = {Sigmoid Loss for Language Image Pre-Training},
  booktitle = {Proceedings of the IEEE/CVF International Conference on Computer Vision (ICCV)},
  pages     = {11941--11952},
  year      = {2023},
}

@article{Oquab2024,
  author  = {Maxime Oquab and
             Timoth{\'{e}}e Darcet and
             Th{\'{e}}o Moutakanni and
             Huy V. Vo and
             Marc Szafraniec and
             Vasil Khalidov and
             Pierre Fernandez and
             Daniel Haziza and
             Francisco Massa and
             Alaaeldin El{-}Nouby and
             Mido Assran and
             Nicolas Ballas and
             Wojciech Galuba and
             Russell Howes and
             Po{-}Yao Huang and
             Shang{-}Wen Li and
             Ishan Misra and
             Michael Rabbat and
             Vasu Sharma and
             Gabriel Synnaeve and
             Hu Xu and
             Herv{\'{e}} J{\'{e}}gou and
             Julien Mairal and
             Patrick Labatut and
             Armand Joulin and
             Piotr Bojanowski},
  title   = {{DINOv2}: Learning Robust Visual Features without Supervision},
  journal = {Transactions on Machine Learning Research},
  year    = {2024},
}

@inproceedings{Hu2022,
  author    = {Edward J. Hu and
               Yelong Shen and
               Phillip Wallis and
               Zeyuan Allen{-}Zhu and
               Yuanzhi Li and
               Shean Wang and
               Lu Wang and
               Weizhu Chen},
  title     = {{LoRA}: Low-Rank Adaptation of Large Language Models},
  booktitle = {Proceedings of the 10th International Conference on Learning Representations (ICLR)},
  year      = {2022},
}

@inproceedings{Panayotov2015,
  author    = {Vassil Panayotov and
               Guoguo Chen and
               Daniel Povey and
               Sanjeev Khudanpur},
  title     = {{LibriSpeech}: An {ASR} Corpus Based on Public Domain Audio Books},
  booktitle = {Proceedings of the IEEE International Conference on Acoustics, Speech and Signal Processing (ICASSP)},
  pages     = {5206--5210},
  year      = {2015},
}

@article{Cao2014,
  author  = {Houwei Cao and
             David G. Cooper and
             Michael K. Keutmann and
             Ruben C. Gur and
             Ani Nenkova and
             Ragini Verma},
  title   = {{CREMA-D}: Crowd-Sourced Emotional Multimodal Actors Dataset},
  journal = {IEEE Transactions on Affective Computing},
  volume  = {5},
  number  = {4},
  pages   = {377--390},
  year    = {2014},
}

@inproceedings{Piczak2015,
  author    = {Karol J. Piczak},
  editor    = {Xiaofang Zhou and
               Alan F. Smeaton and
               Qi Tian and
               Dick C. A. Bulterman and
               Heng Tao Shen and
               Ketan Mayer{-}Patel and
               Shuicheng Yan},
  title     = {{ESC}: Dataset for Environmental Sound Classification},
  booktitle = {Proceedings of the 23rd ACM International Conference on Multimedia},
  pages     = {1015--1018},
  year      = {2015},
}

@inproceedings{Lin2014,
  author    = {Tsung{-}Yi Lin and
               Michael Maire and
               Serge J. Belongie and
               James Hays and
               Pietro Perona and
               Deva Ramanan and
               Piotr Doll{\'{a}}r and
               C. Lawrence Zitnick},
  editor    = {David J. Fleet and
               Tom{\'{a}}s Pajdla and
               Bernt Schiele and
               Tinne Tuytelaars},
  title     = {Microsoft {COCO}: Common Objects in Context},
  booktitle = {Proceedings of the European Conference on Computer Vision (ECCV)},
  pages     = {740--755},
  year      = {2014},
}

@book{Villani2009,
  author    = {Villani, C{\'e}dric},
  title     = {Optimal Transport: Old and New},
  publisher = {Springer},
  year      = {2009},
}

@inproceedings{Borsos2023,
  author    = {Zal{\'{a}}n Borsos and
               Rapha{\"{e}}l Marinier and
               Damien Vincent and
               Eugene Kharitonov and
               Olivier Pietquin and
               Matthew Sharifi and
               Dominik Roblek and
               Olivier Teboul and
               David Grangier and
               Marco Tagliasacchi and
               Neil Zeghidour},
  title     = {{AudioLM}: A Language Modeling Approach to Audio Generation},
  booktitle = {IEEE/ACM Transactions on Audio, Speech, and Language Processing},
  volume    = {31},
  pages     = {2523--2533},
  year      = {2023},
}

@article{Defossez2024,
  author  = {Alexandre D{\'{e}}fossez and
             Laurent Mazar{\'{e}} and
             Manu Orsini and
             Am{\'{e}}lie Royer and
             Patrick P{\'{e}}rez and
             Herv{\'{e}} J{\'{e}}gou and
             Edouard Grave and
             Neil Zeghidour},
  title   = {Moshi: A Speech-Text Foundation Model for Real-Time Dialogue},
  journal = {arXiv preprint arXiv:2410.00037},
  year    = {2024},
}

@article{Tong2024,
  author  = {Peter Tong and
             Ellis Brown and
             Penghao Wu and
             Sanghyun Woo and
             Adithya Iyer and
             Sai Charitha Akula and
             Shusheng Yang and
             Jihan Yang and
             Manoj Middepogu and
             Ziteng Wang and
             Xichen Pan and
             Rob Fergus and
             Yann LeCun and
             Saining Xie},
  editor  = {Amir Globersons and
             Lester Mackey and
             Danielle Belgrave and
             Angela Fan and
             Ulrich Paquet and
             Jakub M. Tomczak and
             Cheng Zhang},
  title   = {Cambrian-1: A Fully Open, Vision-Centric Exploration of Multimodal {LLMs}},
  journal = {arXiv preprint arXiv:2406.16860},
  year    = {2024},
}

@inproceedings{Chen2025listen,
  author  = {Jingyi Chen and
             Zhimeng Guo and
             Jiyun Chun and
             Pichao Wang and
             Andrew Perrault and
             Micha Elsner},
  title   = {Do Audio {LLMs} Really LISTEN, or Just Transcribe? Measuring Lexical vs. Acoustic Emotion Cues Reliance},
  booktitle = {Proceedings of the 19th Conference of the European Chapter of the Association for Computational Linguistics (EACL)},
  year    = {2026},
}

@inproceedings{Yuksekgonul2023,
  author    = {Mert Y\"{u}ksek\-g\"{o}n\"{u}l and
               Federico Bianchi and
               Pratyusha Kalluri and
               Dan Jurafsky and
               James Zou},
  title     = {When and Why Vision-Language Models Behave like Bags-Of-Words, and What to Do About It?},
  booktitle = {Proceedings of the 11th International Conference on Learning Representations (ICLR)},
  year      = {2023},
}

@inproceedings{Thrush2022,
  author    = {Tristan Thrush and
               Ryan Jiang and
               Max Bartolo and
               Amanpreet Singh and
               Adina Williams and
               Douwe Kiela and
               Candace Ross},
  title     = {Winoground: Probing Vision and Language Models for Visio-Linguistic Compositionality},
  booktitle = {Proceedings of the IEEE/CVF Conference on Computer Vision and Pattern Recognition (CVPR)},
  pages     = {5238--5248},
  year      = {2022},
}

@inproceedings{Jain2025,
  author    = {Anubhooti Jain and
               Mayank Vatsa and
               Richa Singh},
  title     = {Words Over Pixels? Rethinking Vision in Multimodal Large Language Models},
  booktitle = {Proceedings of the 34th International Joint Conference on Artificial Intelligence (IJCAI), Survey Track},
  pages     = {10481--10489},
  year      = {2025},
}

@inproceedings{Cuervo2026salad,
  author    = {Santiago Cuervo and
               Skyler Seto and
               Maureen de Seyssel and
               Richard He Bai and
               Zijin Gu and
               Tatiana Likhomanenko and
               Navdeep Jaitly and
               Zakaria Aldeneh},
  title     = {Closing the Gap Between Text and Speech Understanding in {LLMs}},
  booktitle = {Proceedings of the 14th International Conference on Learning Representations (ICLR)},
  year      = {2026},
}

@article{DonskerVaradhan1983,
  author  = {Donsker, Monroe D. and Varadhan, S. R. Srinivasa},
  title   = {Asymptotic Evaluation of Certain {M}arkov Process Expectations for Large Time. {IV}},
  journal = {Communications on Pure and Applied Mathematics},
  volume  = {36},
  number  = {2},
  pages   = {183--212},
  year    = {1983},
}

@article{Kantorovich1958,
  author  = {Kantorovich, Leonid V. and Rubinstein, G. Sh.},
  title   = {On a Space of Totally Additive Functions},
  journal = {Vestnik Leningradskogo Universiteta},
  volume  = {13},
  number  = {7},
  pages   = {52--59},
  year    = {1958},
}
\bibliographystyle{colm2026_conference}

\appendix

\section{Theoretical framework: detailed treatment}
\label{app:theory}
\label{app:proofs}

This appendix contains the detailed formal development outlined in \S\ref{sec:theory}. Each result is preceded by an informal intuitive statement followed by a complete proof.

\subsection{Generalized Mutual Information (GMI)}
\label{app:gmi}
\label{sec:gmi}

\textbf{Overview.} A representation $Z$ might encode some knowledge about a target $Y$ (e.g., the next token), but how much of this knowledge can a given decoder \emph{use}? Standard mutual information $I(Z;Y)$ answers this for a decoder that is optimally matched to the representations it gets. A text-trained LLM is not optimally matched to non-text representations. It was matched to text representations, not speech representations, or image representations. The appropriate measure is the \emph{Generalized Mutual Information} (GMI)~\citep{Merhav1994, Scarlett2020}.

We reuse the notation from \S\ref{sec:formulation}: $q(y|z,c)$ is the decoder's next-token distribution (Eq.~\ref{eq:pipeline}), $Z$ is the representation consumed by the LLM, $C$ is the text context, and $Y$ is the target token. Let $\ell(c,z,y) := \log q(y|z,c)$ be the decoder's log-score (the log-probability it assigns to token $y$ given representation $z$ and context $c$), clipped at a floor $\eta = 1/|\mathcal{V}|$ (where $\mathcal{V}$ is the decoder's output vocabulary) to ensure boundedness.

\textbf{Intuition.} The GMI tells us how well the decoder can distinguish the ``correct'' representation from random alternatives:
\[
\text{GMI} \;=\; \E\bigl[\,\underbrace{\text{Score of the true match}}_{\text{how well the decoder scores the actual } Z} \;-\; \underbrace{\text{LogSumExp over all alternatives}}_{\text{how well random } Z' \text{ would score on average}}\,\bigr].
\]
If the decoder consistently assigns a higher score to the true representation than to a uniformly sampled one, then the GMI is large. If the decoder cannot distinguish which $Z$ generated which $Y$ (assigning the same score to both), then the GMI is zero. This is similar to the Donsker-Varadhan variational representation of KL divergence~\citep{DonskerVaradhan1983}, but applied to the decoder's scoring function instead of an arbitrary test function.

Formally, for a joint distribution $P$ over $(C, Z, Y)$ (which is either the text distribution $\PT$ or the modal distribution $\PM$ from \S\ref{sec:formulation}), the GMI is given by:
\begin{equation}
\GMI_P(q) = \E_{(C,Z,Y) \sim P}\!\left[\ell(C,Z,Y) - \log \E_{Z' \sim P(\cdot|C,Y)}\!\left[e^{\ell(C,Z',Y)}\right]\right].
\label{eq:gmi}
\end{equation}

\begin{theorem}[Accessible rate = GMI; following \citealp{Scarlett2020}]
\label{thm:accessible_rate}
Under i.i.d. sampling from $\PM$ and standard measurability conditions, the maximum rate extractable by the fixed decoder $q$ is $R_{\mathrm{acc}}(\PM, q) = \GMI_{\PM}(q)$.
\end{theorem}

This is a classical result from the mismatched decoding literature; we state it here for completeness and to establish notation. The proof follows directly from the random-coding argument of \citet{Scarlett2020}, Chapters~2--3, with constant-composition codebooks and the GMI decoder metric. The main idea is that the decoder $q$ defines a fixed scoring rule (fixed at inference, regardless of how it was trained), and the maximum rate extractable under this rule, averaged over random codebooks distributed according to $\PM$, approaches $\GMI_{\PM}(q)$.

\subsection{GMI-Wasserstein bound}
\label{app:bound}
\label{sec:bound}

Next we bound the degradation of GMI when the decoder operates on the modal representations $\PM$ rather than the text representations $\PT$ it was trained on. The basic idea is as follows: if modal representations are sufficiently close to text representations, then the decoder should achieve near the same performance on them. How ``sufficiently close'' is determined by two properties of the decoder: how sensitive the decoder is to small perturbations of its inputs, and how large the space of representations is. Theorem~\ref{thm:gmi_wasserstein} relies on three assumptions: the shared marginal (Assumption~\ref{ass:shared}), the Lipschitz condition on the decoder (Assumption~\ref{ass:lipschitz}), and bounded support (Assumption~\ref{ass:bounded}).

\begin{assumption}[Shared marginal]
    \label{ass:shared}
    The text and modal laws have the same prompts and target tokens: $\PM(C,Y) = \PT(C,Y)$.
\end{assumption}

\begin{assumption}[Local Lipschitz log-score]
\label{ass:lipschitz}
There exists $\Llog > 0$ such that for all $(c,y)$ and all $z, z'$ in the typical-state region $\mathcal{Z}$:
$|\ell(c,z,y) - \ell(c,z',y)| \leq \Llog \|z - z'\|$.
\end{assumption}

That is, perturbing the representation by any amount changes the log-probability assigned by the decoder by at most $\Llog$ times that amount. $\Llog$ encodes the sensitivity of the decoder to changes to the input. Neural networks with smooth activations meet this condition locally; we measure $\Llog$ empirically using gradient norms (\S\ref{sec:bound_validation}).

\begin{assumption}[Bounded typical-state region]
\label{ass:bounded}
The typical-state region $\mathcal{Z}$ has diameter $D = \sup_{z,z' \in \mathcal{Z}} \|z - z'\|$, which is finite and $O(\sqrt{d})$ after LayerNorm.
\end{assumption}

This means the representations that are actually used in practice are confined to a bounded region, rather than being distributed throughout $\R^d$. This is precisely what LayerNorm accomplishes: all hidden states are projected onto a hypersphere of radius $\sqrt{d}$.

\begin{theorem}[GMI-Wasserstein bound]
\label{thm:gmi_wasserstein}
Under Assumptions~\ref{ass:shared}, \ref{ass:lipschitz}, and \ref{ass:bounded}:
\begin{equation}
|\GMI_{\PM}(q) - \GMI_{\PT}(q)| \leq \big(1 + e^{\Llog \cdot D}\big) \cdot \Llog \cdot \E_{(C,Y)}\!\left[\Wone\!\big(\PM(Z|C,Y),\, \PT(Z|C,Y)\big)\right].
\label{eq:gmi_bound}
\end{equation}
\end{theorem}

\textbf{Intuition.}
\[
\underbrace{|\text{GMI}_{\text{text}} - \text{GMI}_{\text{modal}}|}_{\text{How much information is lost}} \;\leq\; \underbrace{(1 + e^{\Llog \cdot D})}_{\text{Penalty for distributional shape}} \;\cdot\; \underbrace{\Llog\vphantom{e^D}}_{\text{Decoder sensitivity}} \;\cdot\; \underbrace{\E[\Wone(\PM, \PT)]}_{\text{Distributional distance}}
\]
The difference between how much information the decoder can get from text representations and how much it can get from modal representations is that distance between distributions, multiplied by the strength of the decoder. But we have $\Wone$ of the distributions, not the distance between the means; getting the mean of the modal distribution to the mean of the text distribution is not enough. If the modal distribution has a different shape, or spread, or tails, then the exponential penalty factor $(1 + e^{\Llog D})$  gets them. This is why the text-aligned encoders help to bound (they project on text-like directions, so they decrease both the distance and the shape mismatch), but mean centering would not.

The prefactor $(1 + e^{\Llog D})$ comes from the ``competition term'' in the proof (below). In the ambient space this prefactor is trivially large ($\Llog \cdot D$ ranges from 10 to 150) and the support-restricted bound (\S\ref{app:support}) refines it in terms of the effective support diameter.

\subsubsection{Proof of Theorem~\ref{thm:gmi_wasserstein}}

\textbf{Proof overview.} There are two terms in the GMI: (A) the expected score of the true representation, and (B) the expected score of a random representation. We must prove that the quantity in both (A) and (B) change very little when we replace the text representations with the modal representations. For (A), this is a corollary of the Lipschitz assumption. For (B), we have a potential complication of an exponential of the scores. However, the bounded-diameter assumption ensures that representations are all contained in a bounded region, which prevents large changes due to the exponential.

\textbf{Setup.} We denote $\ell(z) \equiv \ell(c,z,y)$ for fixed $(c,y)$. The GMI functional separates as $\Gamma(P) = A(P) - B(P)$ where:
\begin{align}
A(P) &= \E_{Z \sim P}[\ell(Z)] & &\text{(direct term: expected score of the correct representation)}\\
B(P) &= \log \E_{Z \sim P}[e^{\ell(Z)}] & &\text{(competition term: how well random alternatives score)}
\end{align}

We derive upper bounds for $|A(\PM) - A(\PT)|$ and $|B(\PM) - B(\PT)|$ separately.

\textbf{Step 1: Direct term (how much does the correct representation's score change?).} We would like to bound how much the expected score changes when we go from  $\PT$ to $\PM$. The Kantorovich--Rubinstein theorem~\citep{Kantorovich1958, Villani2009} states that: if a function changes by at most L when you change its input by 1 unit  ($L$-Lipschitz), then the difference in its expectation w.r.t. two distributions is at most L times the Wasserstein distance between the two distributions. Since $\ell$ is $\Llog$-Lipschitz (Assumption~\ref{ass:lipschitz}):
\begin{equation}
|A(\PM) - A(\PT)| = |\E_{\PM}[\ell] - \E_{\PT}[\ell]| \leq \Llog \cdot \Wone(\PM, \PT).
\end{equation}
If modal and text representations are similar (small $\Wone$), and the decoder scores are smooth (small $\Llog$), then the expected score changes by a tiny amount.

\textbf{Step 2: Competition term (how much does the random-alternative score change?).} This is harder because $B(P) = \log \E[e^{\ell(Z)}]$ involves an exponential of the score, which means small differences are magnified. We can use a coupling $\pi$: a way of pairing up samples from $\PM$ and $\PT$ (think of arranging modal representations alongside the text representations they resemble the most). For each pair $(Z_M, Z_T)$, define $\Delta\ell = \ell(Z_M) - \ell(Z_T)$ as the difference in score between paired representations. By the Lipschitz assumption, $|\Delta\ell| \leq \Llog \cdot \|Z_M - Z_T\|$, and by the bounded diameter (Assumption~\ref{ass:bounded}), $|\Delta\ell| \leq \Llog \cdot D$.

Write $m(P) = \E_{Z \sim P}[e^{\ell(Z)}]$. Then:
\begin{align}
\frac{m(\PM)}{m(\PT)} &= \E_\pi\!\left[\frac{e^{\ell(Z_M)}}{m(\PT)}\right] = \E_\pi\!\left[\frac{e^{\ell(Z_T)}}{m(\PT)} \cdot e^{\Delta\ell}\right].
\end{align}
The first term, $e^{\ell(Z_T)}/m(\PT)$, is a probability density under $\PT$ (it sums to 1). The second term, $e^{\Delta\ell}$  accounts for the change in the score when moving from $Z_T$ to $Z_M$. Since $|e^x - 1| \leq |x| \cdot e^{|x|}$ for bounded $|x| \leq \Llog D$:
\begin{equation}
\left|\frac{m(\PM)}{m(\PT)} - 1\right| \leq e^{\Llog D} \cdot \E_\pi[|\Delta\ell|] \leq e^{\Llog D} \cdot \Llog \cdot \E_\pi[\|Z_M - Z_T\|].
\end{equation}
Using $|\log(1+x)| \leq |x|$ for $|x| < 1$ (applies when $\Wone$ is small enough that the ratio stays near 1):
\begin{equation}
|B(\PM) - B(\PT)| = \left|\log \frac{m(\PM)}{m(\PT)}\right| \leq e^{\Llog D} \cdot \Llog \cdot \Wone.
\end{equation}

The $e^{\Llog D}$ term is the cost of the exponential: the competition term is more unstable to distributional shift than the direct term, because the former contains $e^{\ell}$ while the latter contains $\ell$. 

\textbf{Step 3: Combine.} The overall GMI change is at most the sum of the direct and competition terms:
\begin{equation}
|\Gamma(\PM) - \Gamma(\PT)| \leq \underbrace{\Llog \cdot \Wone}_{\text{direct}} + \underbrace{e^{\Llog D} \cdot \Llog \cdot \Wone}_{\text{competition}} = (1 + e^{\Llog D}) \cdot \Llog \cdot \Wone.
\end{equation}
The competition term dominates: it is $e^{\Llog D}$ times larger than the direct term. This is why  the prefactor in the bound is so large and why the support-bounded corollary (which reduces $D$)  is of practical value.

\textbf{Step 4: Average over $(C,Y)$.} Everything above was for a fixed prompt $C$ and target $Y$. By Assumption~\ref{ass:shared} ($\PM(C,Y) = \PT(C,Y)$), the same prompts and targets appear under both laws, so we can simply average over all $(C,Y)$ pairs to obtain the final bound.

When is the bound informative? The bound is a perturbation result: it is tightest when the mismatch $\Wone$ is moderate relative to the decoder's sensitivity. In the large-mismatch regime ($\Llog \cdot \Wone \gg 1$), the prefactor $(1 + e^{\Llog D})$ makes the right-hand side exceed the maximum possible GMI, and the bound becomes vacuous. In this regime, the bound correctly predicts that collapse occurs (the GMI ceiling is far below the matched-decoder rate) but cannot quantify how much information survives. Empirically, speech models operate in this regime ($\Llog \cdot \Wone = 19$--$162$), which is why the bound is loose for speech but informative for vision ($\Llog \cdot \Wone = 13$--$54$). The support-restricted corollary below partially addresses this by replacing the ambient diameter $D$ with the smaller effective diameter $D_\text{eff}$, reducing the prefactor from $\sim e^{150}$ to $\sim e^{7}$.

\subsection{Support-restricted bound}
\label{app:support}

\textbf{Support-restricted bound.}
\label{cor:support}
If $\PM(Z|c,y)$ and $\PT(Z|c,y)$ are supported on some subset $\mathcal{S} \subseteq \mathcal{Z}$ such that $D_{\mathcal{S}} = \mathrm{diam}(\mathcal{S}) \leq D$, then $D$ may be replaced by $D_{\mathcal{S}}$ in Theorem~\ref{thm:gmi_wasserstein}.

\textbf{Proof.} If both marginals are supported in $\mathcal{S}$, any coupling $\pi$ will also be supported on $\mathcal{S} \times \mathcal{S}$, so $\|Z_M - Z_T\| \leq D_{\mathcal{S}}$ for $\pi$-a.e.\ $(Z_M, Z_T)$. Step~2 above then gives that $|\Delta\ell| \leq \Llog \cdot D_{\mathcal{S}}$, which is the improved constant factor.

\textbf{Why this matters.} The ambient diameter $D$ makes the prefactor $(1 + e^{\Llog D})$ extremely large ($\Llog \cdot D$ ranges from 10 to 150), while representations in modal and text space frequently lie in a far smaller effective volume. We estimate the effective support diameter $D_{\text{eff}}$ from the participation ratio of the pooled representation covariance  (\S\ref{sec:bound_validation}), yielding $\Llog \cdot D_{\text{eff}} \in [3, 7]$ and a moderate prefactor ($\sim$20--1100).

\subsection{Per-type prediction}
\label{app:per_type}

The theory predicts the overall trend: text-correlated information types (lexical) should strengthen or hold steady through the LLM, while non-text types (speaker, emotion, acoustic) should weaken. This qualitative prediction holds across all five models.

A stronger prediction would be that the degree of weakening for each information type matches its share of modality-specific variance. This does not hold at the level of individual types (Spearman $\rho = -0.40$, $p = 0.60$ across the four speech information types in Ultravox). The likely cause is the content-dependent $\Llog$: the Lipschitz constant varies $2.5\times$ across audio types ($\Llog = 14.58$ for speech vs.\ $5.92$ for environmental sounds in Ultravox), acting as an additional driver of degradation independent of geometric alignment. A per-type bound would need to incorporate this content-dependence, replacing the scalar $\Llog$ with a type-conditional $\Llog(\tau)$.

\subsection{Probe-decoder asymmetry}
\label{app:asymmetry}
\label{sec:asymmetry}

\begin{theorem}[Probe penalty]
\label{thm:probe}
Let $h(a|z)$ be a linear probe with $\log h$ being $L_h$-Lipschitz in $z$. Then:
$|\E_{\PM}[\log h(A|Z)] - \E_{\PT}[\log h(A|Z)]| \leq L_h \cdot \Wone(\PM, \PT)$.
\end{theorem}

\textbf{Proof.} Identical to Step~1 of the Theorem~\ref{thm:gmi_wasserstein} proof, replacing $\ell$ with $\log h$ and $\Llog$ with $L_h$. Also there is no competition term for probes (no random-coding argument), so the exponential prefactor does not appear.

\textbf{The key asymmetry.} A linear probe is a simple model: one matrix multiplication followed by a softmax. Its sensitivity to distributional shift ($L_h$) is therefore simply the norm of the weights of that matrix. The LLM decoder is a deep transformer with billions of parameters, and its sensitivity ($\Llog$) reflects that depth. Empirically, $\Llog / L_h \approx 30\times$ (\S\ref{sec:bound_validation}).

Now consider the same distributional shift with modal representations replacing text representations. The probe's bound says degradation $\leq L_h \cdot \Wone$ (small). The decoder's bound says degradation $\leq (1 + e^{\Llog D}) \cdot \Llog \cdot \Wone$ (large). Same shift, different impact. This is the formal basis for the ``present but inaccessible'' phenomenon: a probe can still read the information because it is a simple, low-sensitivity instrument, but the decoder's complex scoring rule is far more sensitive to the distributional mismatch and cannot use it.

\section{Experimental details}
\label{app:details}

\subsection{Hook points and probing hyperparameters}

\begin{table}[h]
\centering
\scriptsize
\setlength{\tabcolsep}{2pt}
\begin{tabular}{lp{2.5cm}p{2.8cm}p{2.8cm}p{2.5cm}}
\toprule
\textbf{Model} & \textbf{Encoder} & \textbf{Adapter} & \textbf{LLM-16} & \textbf{LLM-final} \\
\midrule
Ultravox & \texttt{audio\_tower.\allowbreak{}layer\_norm} & \texttt{multi\_modal\_projector.\allowbreak{}ln\_post} & \texttt{language\_model.\allowbreak{}model.layers.16} & \texttt{language\_model.\allowbreak{}model.norm} \\
\midrule
Qwen2-Audio & \texttt{audio\_tower.\allowbreak{}layer\_norm} & \texttt{multi\_modal\_projector} & \texttt{language\_model.\allowbreak{}model.layers.16} & \texttt{language\_model.\allowbreak{}model.norm} \\
\midrule
LLaVA & \texttt{...vision\_tower. ...post\_layernorm} & \texttt{model.mm\_projector} & \texttt{...language\_model. layers.16} & \texttt{...language\_model. norm} \\
\bottomrule
\end{tabular}
\caption{Hook points for representation extraction.}
\label{tab:hooks}
\end{table}

Logistic regression with $\ell_2$ regularization ($C = 1.0$, scikit-learn default). Five random seeds (42, 43, 44, 45, 46). 80/20 stratified train/test split. Input features $z$-score normalized. Mean $\pm$ std reported.

\subsection{Lipschitz estimation}
$\Llog$ estimated via the 2-nearest-neighbor gradient norm method over $n = 1{,}000$ samples. For each sample $z_i$, we compute $\|\nabla_z \log q(y_i | z_i, c_i)\|$ via autograd. The 95th percentile of the resulting distribution serves as the $\Llog$ estimate.

\subsection{Extraction prompts}
\label{app:prompts}

For the representation extraction experiments, we provide each model with a minimal, task-agnostic input to prevent the hidden states from being biased toward a specific downstream task. The representations are always extracted from the intermediate hook locations, not from the output.

\textbf{Speech models.}
Ultravox receives the audio placeholder token \texttt{<|audio|>} wrapped in the model's chat template (system/user/assistant turns). Qwen2-Audio requires a text prompt alongside audio; we use \texttt{"Describe this audio."} in the user turn.

\textbf{Vision models.}
LLaVA receives \texttt{"<image>\textbackslash{}nDescribe this image."} as the user message. Both Prismatic VLMs receive \texttt{"Describe this image."} formatted via the Vicuna chat template.

\textbf{Text baselines.}
Text baselines (Llama-3.1-8B for speech models, Vicuna-7B for vision models) process raw transcripts or captions with no additional prompt, using the model's default tokenization.

\subsection{LoRA experiment details}
\label{app:lora}

The emotion LoRA (\S\ref{sec:implications}) uses standard causal LM fine-tuning on Ultravox with a forced-choice emotion detection objective on CREMA-D.

\textbf{LoRA configuration.}
Rank $r{=}16$, $\alpha{=}32$, dropout $0.05$, target modules: \texttt{q\_proj}, \texttt{k\_proj}, \texttt{v\_proj}, \texttt{o\_proj}.

\textbf{Training.}
AdamW optimizer, learning rate $10^{-4}$, weight decay $0.01$, gradient accumulation $8$ steps (effective batch size $8$), gradient clipping at max norm $1.0$, mixed precision (\texttt{bfloat16}), gradient checkpointing enabled. Five epochs with early stopping (patience $2$, monitoring validation loss). Best checkpoint at epoch~5 (val loss $0.0842$, token accuracy $96.7\%$).

\textbf{Prompt template.}
Each training sample uses a three-turn chat format:
\begin{itemize}
\setlength\itemsep{0pt}
\item \emph{System}: ``You are an expert speech analyst. When presented with audio, identify the emotion expressed by the speaker. Always prioritize what you HEAR in the audio.''
\item \emph{User}: \texttt{Audio: <|audio|> \textbackslash{}n\textbackslash{}n QUESTION: What emotion is the speaker expressing? \textbackslash{}n CHOICES: ["anger","disgust","fear","happy","neutral","sadness"] \textbackslash{}n\textbackslash{}n Your answer MUST be exactly one of the CHOICES above, copied verbatim. \textbackslash{}n Output JSON only: \{"answer": "<exact choice>"\}}
\item \emph{Assistant}: \texttt{\{"answer": "<emotion>"\}} (training target)
\end{itemize}
We calculate the loss only on the assistant response tokens, while the input/prompt tokens are masked (with label $-100$).

\textbf{Data split.}
CREMA-D (7{,}442 clips, 91 speakers, 6 emotions). 80/20 stratified split yields 1{,}002 held-out evaluation samples (167 per emotion class), saved deterministically for reproducibility.

\textbf{Baseline evaluation.}
Base Ultravox (no LoRA) achieves $17.3\%$ task accuracy on the held-out split (near chance for 6 classes), defaulting to ``neutral'' for most inputs.

\subsection{Full probe accuracy tables}

Table~\ref{tab:full_probes} reports the full probe accuracies (mean $\pm$ std over 5 seeds) at all four hook points for every model--dataset--information-type combination. The main finding is that probe accuracy for non-text information types remains well above chance even at the LLM's final layer, confirming that the information survives the decoder's processing. It is the decoder's \emph{use} of this information, not its \emph{presence}, that fails.

\begin{table}[h]
\centering
\small
\begin{tabular}{llcccc}
\toprule
\textbf{Model} & \textbf{Info type (Dataset)} & \textbf{Encoder} & \textbf{Adapter} & \textbf{LLM-16} & \textbf{LLM-Final} \\
\midrule
\multirow{5}{*}{Ultravox}
& Lexical (LS)   & $.422 \pm .014$ & $.279 \pm .018$ & $.570 \pm .023$ & $.535 \pm .027$ \\
& Speaker (LS)   & $.721 \pm .029$ & $.605 \pm .031$ & $.628 \pm .019$ & $.556 \pm .019$ \\
& Emotion (CD)   & $.679 \pm .012$ & $.621 \pm .007$ & $.603 \pm .011$ & $.557 \pm .006$ \\
& Speaker (CD)   & $.508 \pm .005$ & $.196 \pm .010$ & $.204 \pm .011$ & $.137 \pm .008$ \\
& Acoustic (E50) & $.760 \pm .033$ & $.647 \pm .041$ & $.645 \pm .035$ & $.581 \pm .032$ \\
\midrule
\multirow{5}{*}{Q2-Audio}
& Lexical (LS)   & $.270 \pm .006$ & $.242 \pm .015$ & $.506 \pm .018$ & $.471 \pm .021$ \\
& Speaker (LS)   & $.980 \pm .004$ & $.961 \pm .006$ & $.665 \pm .022$ & $.589 \pm .022$ \\
& Emotion (CD)   & $.866 \pm .007$ & $.874 \pm .010$ & $.880 \pm .006$ & $.877 \pm .004$ \\
& Speaker (CD)   & $.880 \pm .009$ & $.816 \pm .003$ & $.725 \pm .016$ & $.556 \pm .012$ \\
& Acoustic (E50) & $.955 \pm .005$ & $.946 \pm .011$ & $.962 \pm .004$ & $.965 \pm .003$ \\
\midrule
\multirow{3}{*}{LLaVA}
& Lexical (CO)     & $.629 \pm .011$ & $.620 \pm .010$ & $.675 \pm .004$ & $.669 \pm .006$ \\
& Object cat.\ (CO)  & $.784 \pm .011$ & $.770 \pm .009$ & $.793 \pm .010$ & $.808 \pm .011$ \\
& Super cat.\ (CO)   & $.802 \pm .011$ & $.797 \pm .007$ & $.833 \pm .009$ & $.842 \pm .012$ \\
\midrule
\multirow{6}{*}{Prism.-D}
& Lexical (CO)     & $.608 \pm .007$ & $.637 \pm .011$ & $.652 \pm .011$ & $.664 \pm .011$ \\
& Object cat.\ (CO)  & $.795 \pm .016$ & $.792 \pm .015$ & $.794 \pm .014$ & $.792 \pm .012$ \\
& Super cat.\ (CO)   & $.837 \pm .013$ & $.832 \pm .010$ & $.834 \pm .011$ & $.839 \pm .014$ \\
& Obj.\ count (CO)   & $.581 \pm .017$ & $.594 \pm .015$ & $.611 \pm .015$ & $.614 \pm .016$ \\
& Obj.\ size (CO)    & $.667 \pm .006$ & $.679 \pm .010$ & $.693 \pm .010$ & $.688 \pm .007$ \\
& Spatial spr.\ (CO) & $.470 \pm .013$ & $.465 \pm .025$ & $.475 \pm .010$ & $.487 \pm .008$ \\
\midrule
\multirow{6}{*}{Prism.-S}
& Lexical (CO)     & $.638 \pm .007$ & $.657 \pm .008$ & $.679 \pm .007$ & $.663 \pm .009$ \\
& Object cat.\ (CO)  & $.813 \pm .007$ & $.804 \pm .011$ & $.811 \pm .010$ & $.819 \pm .014$ \\
& Super cat.\ (CO)   & $.837 \pm .013$ & $.832 \pm .017$ & $.849 \pm .010$ & $.851 \pm .010$ \\
& Obj.\ count (CO)   & $.594 \pm .003$ & $.598 \pm .013$ & $.652 \pm .007$ & $.647 \pm .014$ \\
& Obj.\ size (CO)    & $.654 \pm .019$ & $.669 \pm .014$ & $.722 \pm .018$ & $.699 \pm .020$ \\
& Spatial spr.\ (CO) & $.478 \pm .021$ & $.459 \pm .015$ & $.507 \pm .014$ & $.481 \pm .023$ \\
\bottomrule
\end{tabular}
\caption{Full probe accuracy (mean $\pm$ std over 5 seeds) for all models, datasets, hook points, and information types.}
\label{tab:full_probes}
\end{table}

\subsection{Information retention overview}

Table~\ref{tab:overview} summarizes the information retention pattern across all five models. Retention is computed as the ratio of probe accuracy at the LLM's final layer to probe accuracy at the adapter output, normalized by chance. Values above 100\% indicate the LLM amplifies the signal; values below 100\% indicate degradation. The pattern to notice is that non-text aligned encoders exhibit an asymmetric pattern (lexical content is amplified while speaker identity degrades), while text aligned encoders preserve or improve all the information types, simply because the encoder has already filtered out non-text information.

\begin{table}[h]
\centering
\caption{Information retention (adapter $\to$ LLM-final, \%). Values ${>}100$\% indicate the LLM \emph{amplifies} information; ${<}100$\% indicates degradation. Text-aligned encoders retain all types; non-aligned encoders amplify lexical content while speaker identity collapses by up to 39\%. LS = LibriSpeech, CD = CREMA-D, CO = COCO.}
\label{tab:overview}
\small
\begin{tabular}{@{}llcll@{}}
\toprule
Model & Encoder & Lexical & Best non-text & Worst non-text \\
\midrule
\multicolumn{5}{@{}l}{\emph{Encoder not text-aligned}} \\[2pt]
Ultravox & Whisper & 192\% & 92\% (speaker, LS) & 70\% (speaker, CD) \\
Qwen2-Audio & Whisper & 195\% & 102\% (acoustic) & 61\% (speaker, LS) \\
Prismatic-D & DINOv2 & 104\% & 101\% (super cat.) & 100\% (object cat.) \\
\midrule
\multicolumn{5}{@{}l}{\emph{Encoder text-aligned}} \\[2pt]
LLaVA & CLIP & 108\% & 106\% (super cat.) & 105\% (object cat.) \\
Prismatic-S & SigLIP & 101\% & 102\% (obj.\ \& super cat.) & 102\% (obj.\ \& super cat.) \\
\bottomrule
\end{tabular}
\end{table}

\subsection{Detailed information retention}

Table~\ref{tab:retention} provides the full breakdown of information retention per model and information type, with probe accuracies at the adapter and LLM-final layers.

\begin{table}[h]
\centering
\small
\begin{tabular}{llccc}
\toprule
\textbf{Model} & \textbf{Info type (Dataset)} & \textbf{Adapter} & \textbf{LLM-Final} & \textbf{Retention (\%)} \\
\midrule
\multicolumn{5}{@{}l}{\emph{Encoder not text-aligned}} \\[2pt]
\multirow{5}{*}{Ultravox}
& Lexical (LS)   & .279 & .535 & 192 \\
& Speaker (LS)   & .605 & .556 & 92 \\
& Emotion (CD)   & .621 & .557 & 90 \\
& Speaker (CD)   & .196 & .137 & 70 \\
& Acoustic (E50) & .647 & .581 & 90 \\
\midrule
\multirow{5}{*}{Q2-Audio}
& Lexical (LS)   & .242 & .471 & 195 \\
& Speaker (LS)   & .961 & .589 & 61 \\
& Emotion (CD)   & .874 & .877 & 100 \\
& Speaker (CD)   & .816 & .556 & 68 \\
& Acoustic (E50) & .946 & .965 & 102 \\
\midrule
\multirow{3}{*}{Prism.-D}
& Lexical (CO)     & .637 & .664 & 104 \\
& Object cat.\ (CO) & .792 & .792 & 100 \\
& Super cat.\ (CO)  & .832 & .839 & 101 \\
\midrule
\multicolumn{5}{@{}l}{\emph{Encoder text-aligned}} \\[2pt]
\multirow{3}{*}{LLaVA}
& Lexical (CO)     & .620 & .669 & 108 \\
& Object cat.\ (CO) & .770 & .808 & 105 \\
& Super cat.\ (CO)  & .797 & .842 & 106 \\
\midrule
\multirow{3}{*}{Prism.-S}
& Lexical (CO)     & .657 & .663 & 101 \\
& Object cat.\ (CO) & .804 & .819 & 102 \\
& Super cat.\ (CO)  & .832 & .851 & 102 \\
\bottomrule
\end{tabular}
\caption{Information retention from adapter to LLM-final layer (probe accuracy ratio $\times$ 100). Retention ${>}$100\% means the LLM amplifies the signal; ${<}$100\% means degradation. Non-aligned encoders show selective amplification (lexical recovers, speaker degrades); text-aligned encoders retain or slightly improve all types. LS = LibriSpeech, CD = CREMA-D, E50 = ESC-50, CO = COCO.}
\label{tab:retention}
\end{table}

\subsection{Lipschitz constants and Wasserstein distances}

Table~\ref{tab:lipschitz} reports the empirical Lipschitz constants and Wasserstein distances used to evaluate the bound in \S\ref{sec:deep_dive}.

\begin{table}[h]
\centering
\small
\begin{tabular}{llccccc}
\toprule
\textbf{Model} & \textbf{Dataset} & $\Llog$ & $\Llog$ (p95) & $D$ & $\Wone$ & $\Llog \cdot \Wone$ \\
\midrule
\multicolumn{7}{@{}l}{\emph{Speech}} \\[2pt]
\multirow{3}{*}{Ultravox}
& LibriSpeech & 9.08  & 14.58 & 10.3  & 17.8 & 162 \\
& CREMA-D     & 5.82  & 9.70  & 17.0  & --  & -- \\
& ESC-50      & 3.66  & 5.92  & 20.7  & --  & -- \\
\midrule
\multirow{3}{*}{Q2-Audio}
& LibriSpeech & 0.49  & 0.67  & 122.2 & 38.9 & 19.1 \\
& CREMA-D     & 0.59  & 0.75  & 229.9 & --  & -- \\
& ESC-50      & 0.47  & 0.59  & 202.5 & --  & -- \\
\midrule
\multicolumn{7}{@{}l}{\emph{Vision}} \\[2pt]
LLaVA       & COCO & 0.29 & 0.36 & 28.5 & 46.3 & 13.4 \\
Prism.-D    & COCO & 0.40 & 0.49 & 35.7 & 50.4 & 20.2 \\
Prism.-S    & COCO & 1.12 & 1.44 & 16.2 & 47.9 & 53.7 \\
\bottomrule
\end{tabular}
\caption{Lipschitz constants ($\Llog$, mean and p95 of gradient norm distribution), representation diameter ($D$), Wasserstein-1 distance ($\Wone$ at LLM layer 16), and their product. $\Llog$ is content-dependent: it decreases as audio moves away from text (LibriSpeech $>$ CREMA-D $>$ ESC-50). $\Llog \cdot \Wone$ is largest for Ultravox (162, strongest degradation) and smallest for LLaVA (13.4, no degradation). Prismatic-S has $2.8\times$ higher $\Llog$ than Prismatic-D, consistent with the LLM being more responsive to text-aligned representations.}
\label{tab:lipschitz}
\end{table}

\subsection{Audio-token-only pooling}
\label{app:audio_only}

Our standard extraction mean-pools LLM-layer activations over all sequence positions (system prompt, audio tokens, assistant prompt). This could dilute the audio signal. Ultravox's processor exposes the audio token boundaries (\texttt{audio\_token\_start\_idx}, \texttt{audio\_token\_len}), letting us restrict pooling to audio positions only. We re-extracted LLM representations with audio-only pooling and re-ran probes to compare.

Table~\ref{tab:audio_only} shows the result: audio-only pooling produces nearly identical probe accuracies across all datasets and information types, with all deltas within $\pm$2.1\% and most under 1\%. Restricting to audio tokens does not help, and in several cases slightly hurts. This is consistent with the LLM mixing information across sequence positions during processing, so that the audio-derived signal is available at non-audio positions by layer 16.

\begin{table}[h]
\centering
\small
\begin{tabular}{llccc}
\toprule
\textbf{Dataset} & \textbf{Label} & \textbf{All-token} & \textbf{Audio-only} & \textbf{$\Delta$} \\
\midrule
\multicolumn{5}{@{}l}{\emph{LLM hidden layer 16}} \\[2pt]
LibriSpeech & Speaker  & .637$\pm$.026 & .639$\pm$.023 & $+$.002 \\
LibriSpeech & Lexical  & .591$\pm$.009 & .581$\pm$.019 & $-$.011 \\
CREMA-D     & Speaker  & .210$\pm$.006 & .213$\pm$.009 & $+$.003 \\
CREMA-D     & Lexical  & .999$\pm$.001 & .999$\pm$.001 & $+$.000 \\
CREMA-D     & Emotion  & .593$\pm$.004 & .595$\pm$.006 & $+$.002 \\
ESC-50      & Sound    & .681$\pm$.027 & .677$\pm$.031 & $-$.003 \\
\midrule
\multicolumn{5}{@{}l}{\emph{LLM final layer}} \\[2pt]
LibriSpeech & Speaker  & .565$\pm$.012 & .544$\pm$.020 & $-$.021 \\
LibriSpeech & Lexical  & .573$\pm$.024 & .559$\pm$.023 & $-$.014 \\
CREMA-D     & Speaker  & .140$\pm$.009 & .140$\pm$.008 & $-$.000 \\
CREMA-D     & Lexical  & .999$\pm$.001 & .999$\pm$.001 & $-$.000 \\
CREMA-D     & Emotion  & .556$\pm$.008 & .553$\pm$.004 & $-$.004 \\
ESC-50      & Sound    & .613$\pm$.016 & .603$\pm$.020 & $-$.011 \\
\bottomrule
\end{tabular}
\caption{Probe accuracy with all-token vs.\ audio-only pooling at LLM layers (Ultravox). All deltas are within $\pm$2.1\%, most under 1\%. Audio-only pooling does not improve probe accuracy. }
\label{tab:audio_only}
\end{table}

\begin{figure}[h]
\centering
\includegraphics[width=\linewidth]{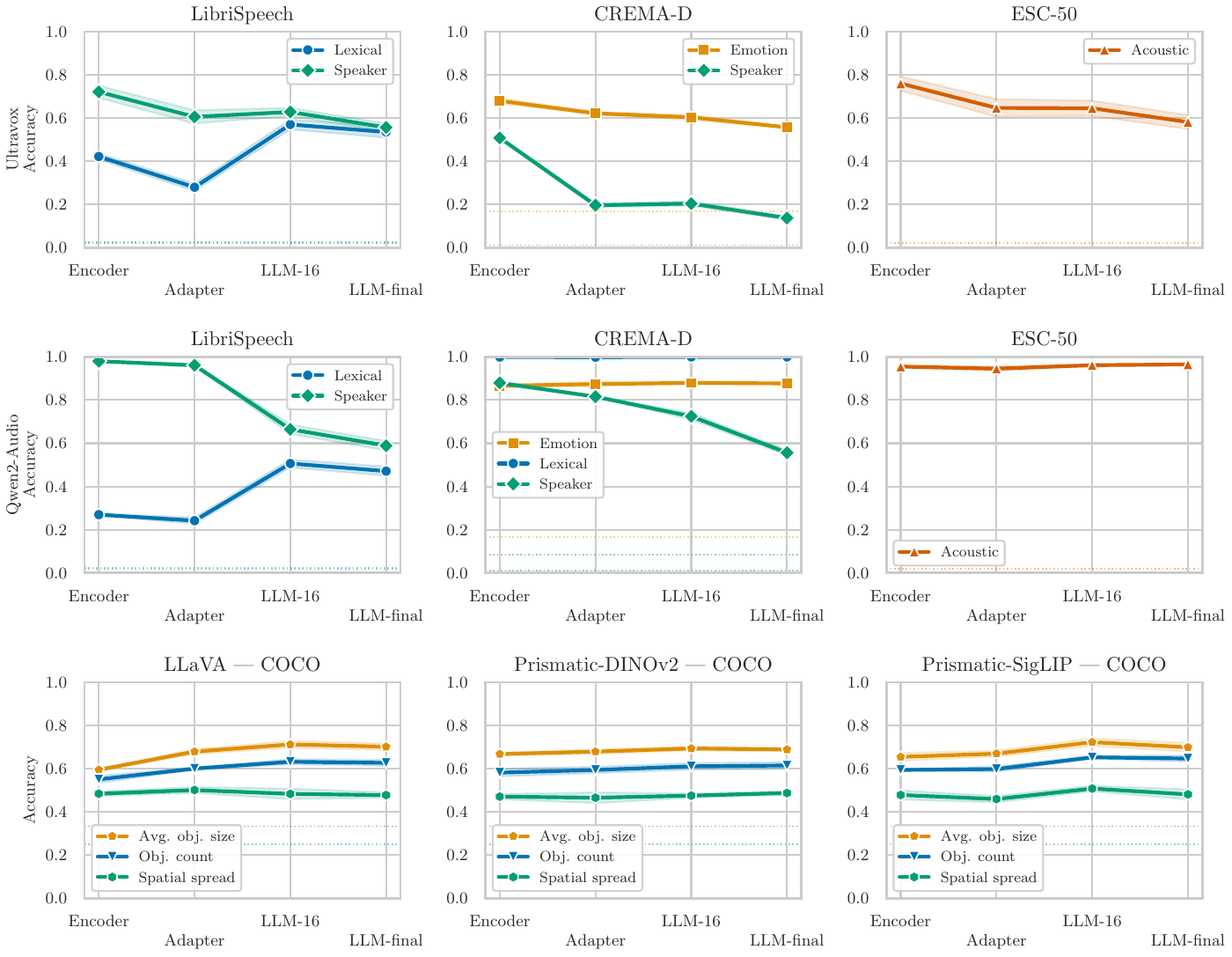}
\caption{Linear probe accuracy at four extraction points (encoder, adapter, LLM layer~16, LLM final) for all five models. For speech models (rows 1--2), lexical accuracy increases through the LLM while speaker identity degrades, illustrating the information accessibility gap. For vision models (row 3), LLaVA and Prismatic-S (text-aligned encoders) show improvement across all information types, while Prismatic-D (non-aligned) shows flat or declining non-textual accuracy. Datasets: LibriSpeech, CREMA-D, ESC-50 (speech); COCO (vision).}
\label{fig:app_trajectory}
\end{figure}

\begin{figure}[h]
\centering
\includegraphics[width=\linewidth]{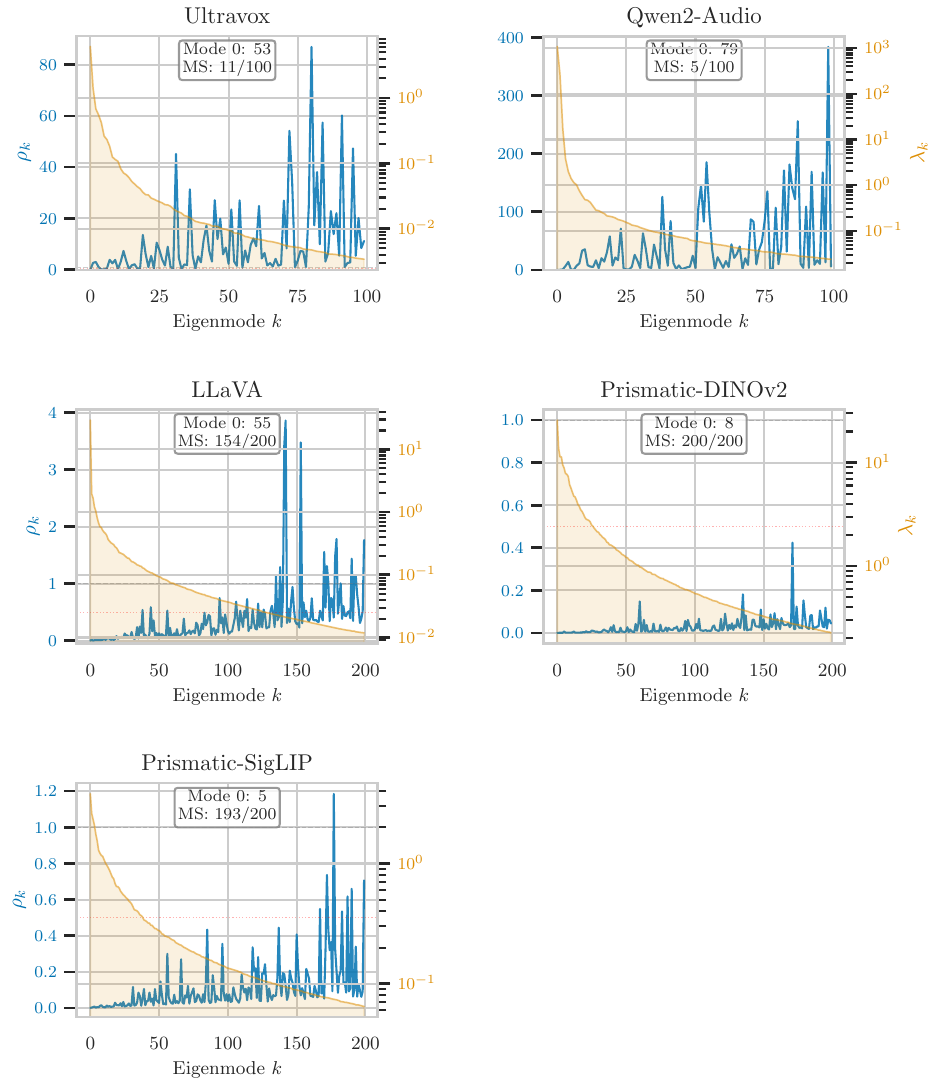}
\caption{Mode alignment profiles for all five models. Each panel shows the top modes (eigenvectors of the adapter-output covariance $\Sigma_M$, sorted by decreasing eigenvalue; 100 for speech models, 200 for vision). Blue (left axis): alignment score $\rho_k$, the ratio of text variance to modal variance along each mode (\S\ref{sec:mode_alignment}). Orange (right axis): eigenvalue $\lambda_k$ (log scale), the amount of adapter variance carried by each mode. Annotations show the percentage of total variance in Mode~0 and the count of modality-specific modes ($\rho_k < 0.5$). For non-aligned encoders, the highest-variance modes have $\rho_k \approx 0$: the adapter's dominant variation is along directions the decoder has never seen. For Prismatic-S (SigLIP, text-aligned), even Mode~0 is text-aligned ($\rho_0 = 0.83$).}
\label{fig:app_mode_alignment}
\end{figure}

The Prismatic comparison (Table~\ref{tab:prismatic}) isolates encoder text-alignment as the causal variable. Both VLMs share the same architecture, adapter, training recipe, and LLM backbone; only the vision encoder differs. SigLIP (text-aligned) shows consistent improvement through the LLM for all information types, while DINOv2 (not text-aligned) shows stagnation or minimal change for non-textual attributes.

\begin{table}[h]
\centering
\small
\begin{tabular}{llccc}
\toprule
& & \multicolumn{3}{c}{\textbf{Probe accuracy}} \\
\cmidrule(lr){3-5}
\textbf{Encoder} & \textbf{Info type} & \textbf{Adapter} & \textbf{LLM-final} & \textbf{$\Delta_{\text{LLM}}$} \\
\midrule
\multicolumn{5}{@{}l}{\emph{Text-describable attributes}} \\[2pt]
\multirow{3}{*}{DINOv2} & Lexical    & .637 & .664 & \textcolor{darkblue}{+4.2\%} \\
                         & Object cat.  & .792 & .792 & $-$0.0\% \\
                         & Super cat.   & .832 & .839 & +0.8\% \\
\midrule
\multirow{3}{*}{SigLIP}  & Lexical    & .657 & .663 & +0.9\% \\
                         & Object cat.  & .804 & .819 & \textcolor{darkblue}{+1.9\%} \\
                         & Super cat.   & .832 & .851 & \textcolor{darkblue}{+2.3\%} \\
\midrule
\multicolumn{5}{@{}l}{\emph{Non-textual attributes}} \\[2pt]
\multirow{3}{*}{DINOv2} & Obj.\ count  & .594 & .614 & +3.4\% \\
                         & Obj.\ size   & .679 & .688 & +1.3\% \\
                         & Spatial spr. & .465 & .487 & +4.7\% \\
\midrule
\multirow{3}{*}{SigLIP}  & Obj.\ count  & .598 & .647 & \textcolor{darkblue}{+8.2\%} \\
                         & Obj.\ size   & .669 & .699 & \textcolor{darkblue}{+4.5\%} \\
                         & Spatial spr. & .459 & .481 & +4.8\% \\
\bottomrule
\end{tabular}
\caption{Prismatic controlled comparison (full results). $\Delta_{\text{LLM}}$ = change from adapter to LLM-final. Top: text-describable attributes (named in captions). Bottom: non-textual attributes (object count, average object size, spatial spread; derived from annotations, not captions). Text baseline accuracy is substantially lower (.415--.495) for non-textual attributes, confirming their visual nature. Same architecture, same LLM, only encoder differs.}
\label{tab:prismatic}
\end{table}

The mode alignment spectrum (Table~\ref{tab:mode_alignment}) reveals which directions carry the mismatch. For each model, we decompose the adapter output covariance into eigenmodes and measure their text alignment. The dominant eigenmode is modality-specific ($\rho_k \ll 1$) for all non-text-aligned encoders, while for SigLIP even the dominant mode is text-aligned ($\rho_k = 0.83$).

\begin{table}[h]
\centering
\small
\begin{tabular}{lcccc}
\toprule
\textbf{Model} & \textbf{Mode 0 $\rho_k$} & \textbf{Mean $\rho_k$ (top 10)} & \textbf{MS modes} & \textbf{MS var.} \\
\midrule
Ultravox    & 0.001  & 1.76  & 11/100  & 63.6\% \\
Qwen2-Audio & 0.010  & 7.21  & 5/100   & 97.6\% \\
LLaVA       & 0.013  & 3.94  & 9/100   & 71.9\% \\
Prismatic-D & 0.034  & 0.087 & 53/100  & 71.0\% \\
Prismatic-S & 0.827  & 1.95  & 14/100  & 24.6\% \\
\bottomrule
\end{tabular}
\caption{Mode alignment summary. MS modes = modes with $\rho_k < 0.5$ (modality-specific). MS var.\ = fraction of top-100 eigenmode variance in MS modes. Mode~0 is modality-specific ($\rho_k \ll 1$) for all non-text-aligned encoders; for SigLIP (text-aligned), even Mode~0 is text-aligned ($\rho_k = 0.83$). Extending to all 4096 modes strengthens the pattern: tail modes are overwhelmingly text-aligned with negligible individual variance, so the MS/TA split is driven by the high-variance modes reported here.}
\label{tab:mode_alignment}
\end{table}

\section{Additional limitations}
\label{app:additional_limitations}

\textbf{Scope.}
We test speech and vision across five models. The framework applies to any multimodal LLM with a text-trained decoder, but video/3D remain untested.

\textbf{Assumption~\ref{ass:shared} (shared marginal).} We assume $\PM(C,Y) = \PT(C,Y)$. This holds when the multimodal model is trained on text-centric tasks (transcription, captioning). Violations add a marginal-shift term independent of representation geometry.

\textbf{Generality of the GMI constraint.} The mechanistic explanation is not architectural: the text-shaped scoring rule of the multimodal LLM decoder is conditioned only on the text-dominating training objective, and holds regardless of the adapter implementation (linear projection, MLP, Q-Former, perceiver resampler, discrete codebook, etc.). The GMI constraint is more broadly relevant for multimodal modelling: in the absence of a training objective that explicitly seeks out each modality, the decoder will remain indifferent to non-text directions.

\end{document}